\documentclass[lettersize,journal]{IEEEtran}
\usepackage{amsmath,amsfonts}
\usepackage{algorithm}
\usepackage{array}
\usepackage[caption=false,font=normalsize,labelfont=sf,textfont=sf]{subfig}
\usepackage{textcomp}
\usepackage{stfloats}
\usepackage{url}
\usepackage{verbatim}
\usepackage{graphicx}
\usepackage{cite}
\makeatletter
\usepackage{multirow}
\usepackage{makecell}
\usepackage{enumerate}
\usepackage{amsthm}
\usepackage{amssymb}
\usepackage{algpseudocode}
\usepackage{bbm}
\usepackage{microtype}
\usepackage{siunitx}
\usepackage{color} 

\newcolumntype{L}[1]{>{\raggedright\let\newline\\\arraybackslash\hspace{0pt}}m{#1}}
\newcolumntype{C}[1]{>{\centering\let\newline\\\arraybackslash\hspace{0pt}}m{#1}}
\newcolumntype{R}[1]{>{\raggedleft\let\newline\\\arraybackslash\hspace{0pt}}m{#1}}

\theoremstyle{definition}	
\newtheorem{assumption}{Assumption}
\newtheorem{problem}{Problem}

\newtheorem{theorem}{Theorem}
\newtheorem{remark}{Remark}

\makeatother

\begin{document}
\title{High-Speed Interception Multicopter Control by Image-based Visual Servoing}

\author{Kun Yang, Chenggang Bai, Zhikun She, Quan Quan
\thanks{Kun Yang, Chenggang Bai, Quan Quan are with the School of Automation Science and Electrical Engineering, Beihang University, Beijing 100191, China (e-mail: yangkun\_buaa@buaa.edu.cn; bcg@buaa.edu.cn; qq\_buaa@buaa.edu.cn). }
\thanks{Zhikun She is with the School of Mathematical Sciences, Beihang University, Beijing 100191, China (e-mail: zhikun.she@buaa.edu.cn). }
\thanks{\textit{Corresponding author: Quan Quan}}
}

\maketitle
\begin{abstract}
	In recent years, reports of illegal drones threatening public safety have increased. For the invasion of fully autonomous drones, traditional methods such as radio frequency interference and GPS shielding may fail. This paper proposes a scheme that uses an autonomous multicopter with a strapdown camera to intercept a maneuvering intruder UAV. The interceptor multicopter can autonomously detect and intercept intruders moving at high speed in the air. 
	The strapdown camera avoids the complex mechanical structure of the electro-optical pod, making the interceptor multicopter compact. However, the coupling of the camera and multicopter motion makes interception tasks difficult. 
	To solve this problem, an Image-Based Visual Servoing (IBVS) controller is proposed to make the interception fast and accurate. Then, in response to the time delay of sensor imaging and image processing relative to attitude changes in high-speed scenarios, a Delayed Kalman Filter (DKF) observer is generalized to predict the current image position and increase the update frequency.
	Finally, Hardware-in-the-Loop (HITL) simulations and outdoor flight experiments verify that this method has a high interception accuracy and success rate. In the flight experiments, a high-speed interception is achieved with a terminal speed of 20\,m/s.
\end{abstract}


Video of the experiments: \url{https://youtu.be/146Km6c30Ww}

\section{Introduction\label{sec:Introduction}}

Advances in micro air vehicles, also known as drones, are exploiting opportunities in industrial, agricultural and military  fields \cite{Lee2021Antisway,floreano2015science}.
The rapid expansion of the drone industry has outpaced regulations for the safe operation of drones, making them representative means of illegal and destructive terror and crime \cite{ritchie2017micro}.
This has spurred the development of anti-drone (or counter-drone) technologies \cite{shi2018anti} to mitigate risks from drone accidents or terrorism, particularly against increasingly autonomous drones.

The drone interceptors can be classified as radio frequency interference \cite{10.1145/3309735, parlin2018jamming} and hard-kill of enemy drones \cite{Dronecatcher2021, chaari2020testing}. 
However, due to the small size and high flexibility of fully autonomous drones \cite{chen2020aerobatic,foehn2022alphapilot}, using these ground-based devices to rapidly deploy and intercept intruders is a challenging task.
A proposed solution is the use of drones with onboard sensors for effective interception, as depicted in Fig. \ref{fig:Interception-with-drones}.

\begin{figure}
	\centering{}\includegraphics{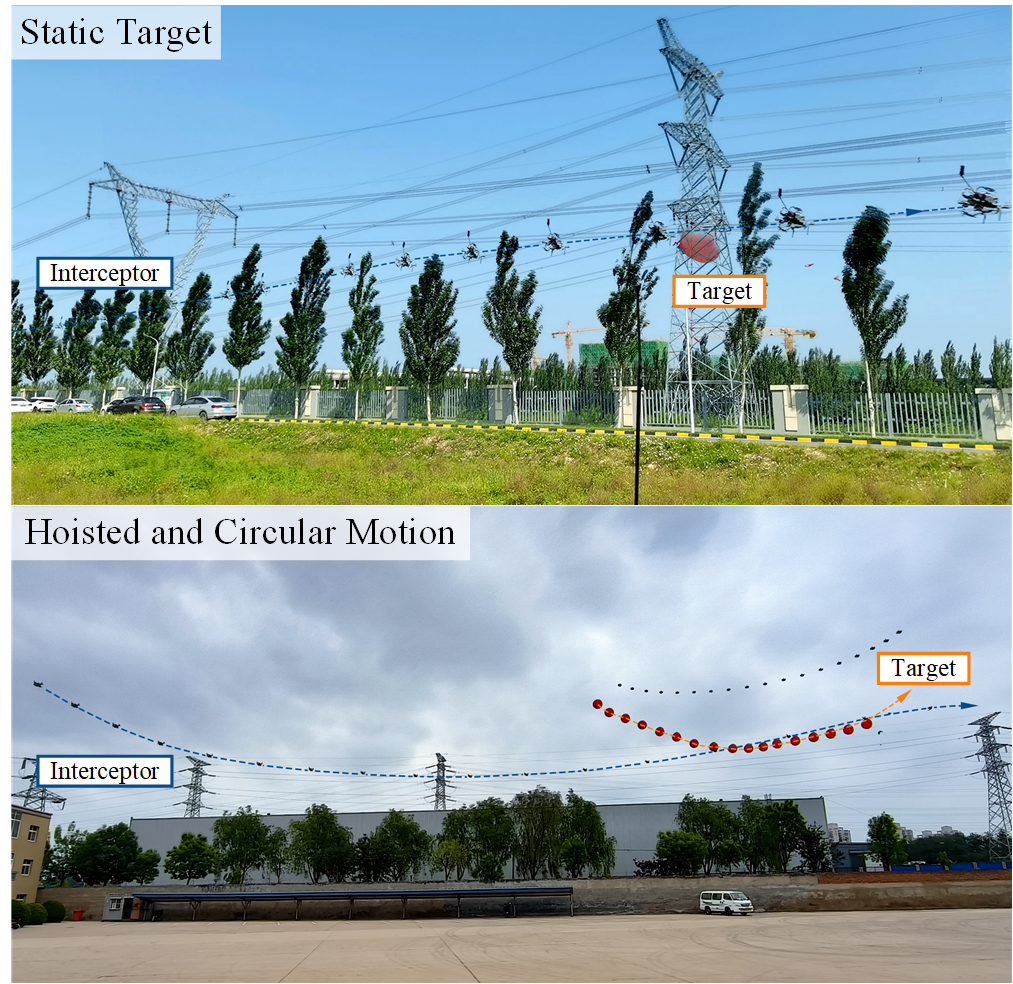}\caption{\label{fig:Interception-with-drones}Interception with drones.}
\end{figure}

\vspace*{-0.2cm}
\subsection{Problem Statement and Challenges}
Our previous research developed an Image-Based Visual Servo (IBVS) algorithm with a forward-looking monocular camera for real flight experiments at \SI{5}{m/s} \cite{yang2020autonomous}. This approach, however, struggles with high-speed interception tasks due to its reliance on the small-angle approximation, which fails at higher pitch angles observed during rapid movements, reaching up to \ang{50}.

High-speed multicopter movement highlights significant challenges, particularly in camera imaging and processing. {Delays around \SI{100}{ms} lead to notable control errors, up to \SI{2}{m} at speeds like \SI{20}{m/s}, challenging the assumption of flawless image processing.} Moreover, the camera's imaging frequency falls short of the requirements for effective control.
To solve the high-speed interception problem, two main research challenges are concluded:

\begin{enumerate}[C-1:]
\item The interceptor's high-speed and aggressive maneuverability necessitate a comprehensive attitude model and a redesigned controller that can handle extensive state space. This involves working within the three-dimensional special orthogonal group space, ${\mathsf{SO}}\left(3\right)$, requiring the development of a controller for large maneuver interception tasks.

\item The delayed, low frame rate, and easy loss of image feedback during the high-speed motion of the interceptor will lead to large interception errors. Accurate, real-time estimation of target coordinates on image is thus essential.
\end{enumerate}

\subsection{Related Work}

\begin{table*}[]
	\caption{\label{tab:Autonomous-interception-and-racing}Autonomous interception and racing of drones}
	\resizebox{1.0\linewidth}{!}{
	\renewcommand{\arraystretch}{1.8}
	\begin{tabular}{m{1.1cm}L{2cm}m{3cm}m{1.5cm}m{2cm}l}
	\hline
	Category                                   & Method                                                   & Feature                                                                                                                                                                        & Target speed                                                                                        & Interceptor speed                                                    & \multicolumn{1}{c}{Pipeline}                 \\ \hline
	\multirow{2}{*}{\makecell[l]{Location\\ estimation- \\ based}} & Classical perception-planning-control pipeline \cite{stasinchuk2021multi, vrba2022autonomous}          & Second place in Challenge 1 in MBZIRC 2020                                                                                                                                     & \makecell[l]{\SI{0}{m/s};\\ \SI{7}{m/s}}                                                & \makecell[l]{2\,-\,3\,\SI{}{m/s};\\ Almost no\\ mobility} & \raisebox{-.5\height}{\includegraphics[scale=0.8]{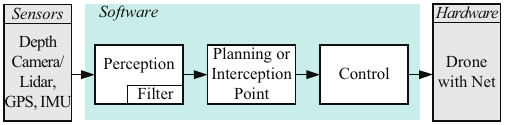}}            \\ \cline{2-6} 
											   & Optimal terminal-velocity-control guidance \cite{Lin2022}              & Make further improvement on the basis of Challenge 1 in MBZIRC 2020 (first place)                                                                                             & \SI{10}{m/s}                                                                                               & \SI{10}{m/s}                                                                & \raisebox{-.5\height}{\includegraphics[scale=0.8]{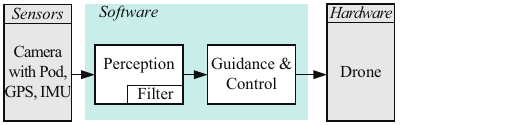}}           \\ \hline
	\multirow{2}{*}{\makecell[l]{Learning-\\ based} }            & Imitation learning and low-level control \cite{Deep2020}                & The neural network generates a minimum-jerk trajectory and the controller tracks it. First to demonstrate zero-shot sim-to-real transfer on drone racing & \makecell[l]{Simulation: \\ move slowly;\\ Real flight: \\ static narrow \\ gates} & \makecell[l]{4\,-\,12\,\SI{}{m/s};\\ 1\,-\,3\,\SI{}{m/s}}            & \raisebox{-.5\height}{\includegraphics[scale=0.8]{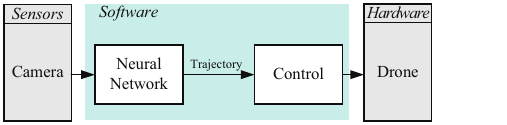}} \\ \cline{2-6} 
											   & End to end, imitation learning \cite{fu2022learning}                          & The neural network predicts commands directly from vision. Inputs are a sequence of images and drone state but no globally-consistent position     &  \makecell[l] {Simulation: \\ static gates}                                                                          & \textgreater{}\SI{10}{m/s}                                                               & \raisebox{-.5\height}{\includegraphics[scale=0.8]{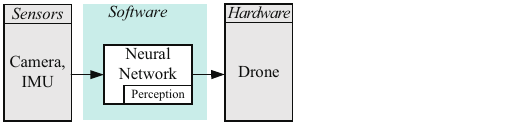}}  \\ \hline
	\multirow{2}{*}{\makecell[l]{IBVS-\\ based}}                      & Pseudo-linear Kalman filter and 3-D helical guidance law \cite{9946840} & A new 3-D helical guidance law is designed based on the theoretical conclusion of the observability enhancement                                                                & \SI{5}{m/s}                                                                                                & Helical tracking                                                     & \raisebox{-.5\height}{\includegraphics[scale=0.8]{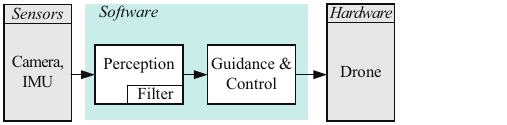}}         \\ \cline{2-6} 
											   & Perception and IBVS controller \cite{yang2020autonomous}                          & Use only a camera for autonomous UAV interception tasks, no global position is required                                                                                   & \makecell[l]{Simulation: \\ \SI{10}{m/s};\\ Real flight: \\ \SI{2}{m/s}}                      & \makecell[l]{\SI{12}{m/s};\\ \SI{5}{m/s}}                & \raisebox{-.5\height}{\includegraphics[scale=0.8]{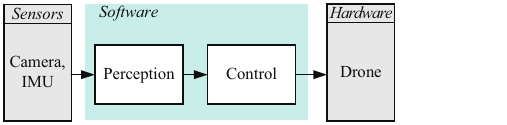}}          \\ \hline
	\end{tabular}
	}
\vspace{-0.3cm}
\end{table*}

\subsubsection{Review on Autonomous Interception and Racing}
\begin{figure}
	\centering{}\includegraphics[scale=0.9]{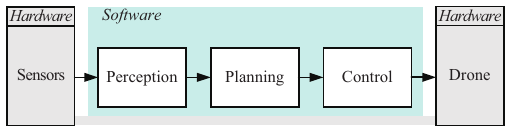}\caption{\label{fig:autonomous-concept}{A classic architecture for an autonomous interception and racing system.}}
\vspace{-0.4cm}
\end{figure}

Drone racing's objective of precise gate traversal aligns with autonomous interception goals, meriting inclusion in our review \cite{hanover2023autonomous}. {Typical autonomous interception and racing systems, as shown in Fig. \ref{fig:autonomous-concept}, consist of hardware (drones, sensors) and software (perception, planning, control modules). \textls[-9]{Research in this field falls into three categories: location estimation-based, learning-based, and IBVS-based methods, summarized in Table \ref{tab:Autonomous-interception-and-racing}.}}

Location estimation-based methods, using a perception-planning-control pipeline, were evident in MBZIRC 2020, with teams implementing such strategies to intercept both stationary and moving targets \cite{stasinchuk2021multi, vrba2022autonomous, Lin2022, 8746787}. These methods depend heavily on the estimation accuracy, planning rationality, and control accuracy for interception accuracy.

Learning-based methods involve neural networks replacing traditional perception, planning, and control modules. They focus on learning trajectories and predicting commands from visual inputs, offering robustness against system latencies and sensor noise \cite{Deep2020, fu2022learning}. Despite their potential, the complexity of real-world data collection limits their application in actual flight scenarios.

IBVS-based methods eliminate the planning module to reduce errors from inaccurate location estimation and planning. These methods, focusing on bearing measurements, have led to innovations like 3-D helical guidance laws \cite{9946840} and the use of virtual cameras for tracking targets \cite{8628313, 10044962}. This paper particularly explores the IBVS controller within the ${\mathsf{SO}}\left(3\right)$ space for effective large maneuver interceptions.

\subsubsection{Review on Delayed State Estimation}
{Indoor and outdoor localization systems frequently encounter measurement delays and losses, especially in Non-Line-of-Sight (NLOS) situations. This is akin to the challenges posed by visual sensor delays in drone technology. In indoor systems, strategies like Finite Impulse Response (FIR) filters have been developed, where lost measurements are predicted based on previous estimates \cite{7565735, GUAN2018161}. For outdoor systems, particularly those using GPS in obstacle-rich environments, Variational Bayesian (VB)-based Kalman filters are employed, utilizing Bernoulli variables for time delay and state variable estimation \cite{9477131}.}

{Despite these developments, modeling visual sensor delays in state estimation remains a challenge. Our approach, inspired by existing research, recalculates states over the delay period to address this \cite{9362168}.} To minimize computational demands, we update the covariance matrix as if there were no delay, adjusting it with an additional correction term when the delayed measurement arrives. This results in suboptimal estimates during the delay but optimal ones post-delay. We introduce the Delayed Kalman Filter (DKF) to manage delayed image measurements and real-time multicopter feedback, achieving high frame rate real-time image state estimates. This method, validated through simulations and experiments, effectively addresses issues related to delayed, low frame rate, and easily lost image feedback.

\subsection{Proposed Methods and Main Contributions}
This paper introduces a novel approach for using multicopters to intercept aerial targets. We propose a controller on ${\mathsf{SO}}\left(3\right)$ based on IBVS for high-speed, large-maneuver interceptions, and a DKF observer to counteract image delays in high-speed situations.

The contributions of this paper are summarized in the following. (i) An IBVS scheme is applied to high-speed interception, which solves the problem of coupling between aircraft motion and feature point imaging when the camera is fixedly connected to the aircraft body. In addition, this scheme is suitable for any installation angle and installation position of the camera. 
(ii) A delay filtering scheme is proposed to remedy the defects of delayed, low frame rate, and easy loss of image feedback. 
{(iii) The experiments verified the effectiveness of the algorithm. These experiments are rare in showcasing a multicopter flying at high speeds exceeding \SI{20}{m/s}, with a maximum pitch angle reaching \ang{50}, relying solely on its onboard perception and control.}

\subsection{Outline}

\begin{figure}
	\centering{}\includegraphics[scale=0.95]{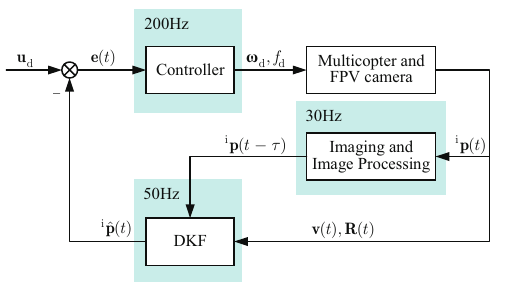}\caption{\label{fig:model-controller-and-observer}Relationship among visual servo model, controller, and state observer.}
\vspace{-0.3cm}
\end{figure}

The overall system architecture is shown in Fig. \ref{fig:model-controller-and-observer}.
In Section \ref{sec:Problem-Statement}, the high-speed multicopter interception problem is formulated and divided into the optimal state estimation problem and the attitude loop controller design problem. {Then, the attitude loop controller is designed based on the Lyapunov control theory and its stability is shown in Section \ref{sec:high-speed-interception-control}, and the DKF is proposed as a state observer to solve the optimal state estimation problem in Section \ref{sec:Delayed-Kalman-Filter}.} Finally, necessary simulations and real flight experiments are shown in Sections \ref{sec:HITL-simulation} and \ref{sec:flight-experiment}, and a conclusion is drawn in Section \ref{sec:conclusion}.

\section{Model and Problem Formulation\label{sec:Problem-Statement}}

In this section, the 6-Degree-of-Freedom (6-DOF) rigid model of multicopter and the IBVS model are proposed. With them, the interception controller design problem is formulated.  
Furthermore, the sensor measurement model is proposed and the state observer problem is conducted.

The relationship among the multicopter and IBVS models, the controller, and the state observer is shown in Fig. \ref{fig:model-controller-and-observer}.
The interception controller calculates the control signal by the image error and the multicopter state, and the DKF state observer estimates the current image features based on the multicopter state and the delayed image features. 
The specific forms of the three parts are introduced separately to achieve precise interception control under high-speed conditions.

\subsection{Coordinate Systems and Models}

\begin{figure}
\centering{}\includegraphics[scale=0.95]{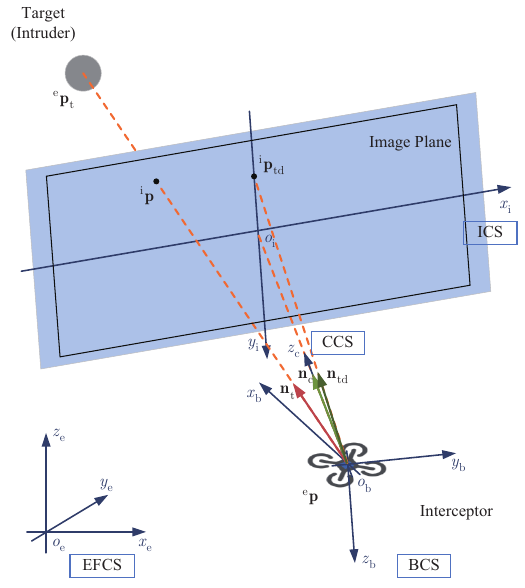}\caption{\label{fig:Description-of-the-interception-problem}Description of the interception problem. 
${^{{\rm {e}}}{\bf {p}}}$ and
${^{{\rm {e}}}{{\bf {p}}_{{\rm {t}}}}}$ respectively represent
the position of the intercepting multicopter and the target in EFCS; 
$^{{\rm {i}}}{\bf {p}}={\big[{{p_{{x_{{\rm {i}}}}}} \enspace {p_{{y_{{\rm {i}}}}}}}\big]^{{\rm {T}}}}$ is the coordinate of the target in ICS; 
${{{\bf {n}}_{{\rm {t}}}}}$ represents the target unit vector along the Line of Sight (LOS) in EFCS;
${{{\bf {n}}_{{\rm {c}}}}}$ represents the optical axis unit vector in EFCS; 
${\bf{n}}_{{\rm{td}}}$ is defined as the unit vector of the designed LOS in EFCS, called the designed LOS vector; ${^{{\rm {i}}}{\bf {p}}}_{{\rm {td}}}$ is the intersection of ${\bf{n}}_{{\rm{td}}}$ with the image plane in ICS, called the intercept point. 
}
\end{figure}

\subsubsection{Coordinate Systems}

There are five coordinate systems used in this paper, as shown in Fig. \ref{fig:Description-of-the-interception-problem}, including
\begin{itemize}
	\item ${\left\{ {\rm {e}}\right\} =\left\{ {{o_{{\rm {e}}}}\operatorname{-}{x_{{\rm {e}}}}{y_{{\rm {e}}}}{z_{{\rm {e}}}}}\right\} }$
	is the \textit{Earth-Fixed Coordinate System (EFCS)}, which is used to represent the position
	and velocity of interceptors and targets in the global inertial system; 
	\item {${\left\{ {\rm {b}}\right\} =\left\{ {{o_{{\rm {b}}}}\operatorname{-}{x_{{\rm {b}}}}{y_{{\rm {b}}}}{z_{{\rm {b}}}}}\right\} }$
	is the \textit{Body Coordinate System (BCS)}, which is used to represent the variables
	of the current attitude of interceptors;}
	\item ${\left\{ {\rm {c}}\right\} =\left\{ {{o_{{\rm {c}}}}\operatorname{-}{x_{{\rm {c}}}}{y_{{\rm {c}}}}{z_{{\rm {c}}}}}\right\} }$
	is the \textit{Camera Coordinate System (CCS)};
	\item ${\left\{ {\rm {i}}\right\} =\left\{ {{o_{{\rm {i}}}}\operatorname{-}{x_{{\rm {i}}}}{y_{{\rm {i}}}}}\right\} }$
	is the \textit{Image Coordinate System (ICS)}, used to express the image position
	of target features in the first-person view of cameras.
\end{itemize}

Matrix ${{\bf {R}}_{{\mathcal {F}}_1}^{{\mathcal {F}}_2}}$ represents
a rotation matrix from coordinate system ${{\mathcal {F}}_1}$ to ${{\mathcal {F}}_2}$, ${{\bf {R}}_{{\mathcal {F}}_1}^{{\mathcal {F}}_2}\in \mathsf{SO}\left(3\right)}$,
where the 3-D special orthogonal group is given by
\[
\mathsf{SO}\left(3\right)=\left\{ {{\bf {R}}\in{\mathbb{R}^{3\times3}} \mid {{\bf {R}}^{{\rm {T}}}}{\bf {R}}={{\bf {I}}},\det\left({\bf {R}}\right)=1}\right\}.
\]

\subsubsection{Multicopter Model}
{The multicopter (interceptor) is modeled as a rigid body of mass $m$ with a rotor drag proportional to its velocity \cite{KAI20178189, scirobotics.adg1462}. Its flight control rigid model is summarized as
\begin{equation}
\left\{ \begin{aligned}^{{\rm {e}}}\dot{{\bf {p}}} & = {^{{\rm {e}}}{\bf {v}}}\\
^{{\rm {e}}}\dot{{\bf {v}}} & ={{\bf {g}}}+\frac{1}{m}\left({}^{{\rm {e}}}{\bf {f}} + {}^{{\rm {e}}}{\bf {f}}_{\rm{drag}}\right) \\
\dot{{\bf {R}}}_{{\rm {b}}}^{{\rm {e}}} & ={\bf {R}}_{{\rm {b}}}^{{\rm {e}}}{\left[{^{{\rm {b}}}{\bf {\omega}}}\right]_{\times}}\\
{\bf {J}}\cdot{}^{{\rm {b}}}{\dot {\bf{\omega}}} & =-^{{\rm {b}}}{\bf {\omega}}\times\left({{\bf {J}}\cdot{}^{{\rm {b}}}{\bf {\omega}}}\right)+{{\bf {G}}_{a}}+{\bf {\tau}}
\end{aligned}
\right.\label{eq:rigid-body-model}
\end{equation}}
where ${^{{\rm {e}}}{\bf {p}}={\big[{{p_{{x_{{\rm {e}}}}}} \enspace {p_{{y_{{\rm {e}}}}}} \enspace {p_{{z_{{\rm {e}}}}}}}\big]^{{\rm {T}}}}}$ is the position of the multicopter under EFCS; 
${^{{\rm {e}}}{\bf {v}}={\big[{{v_{{x_{{\rm {e}}}}}} \enspace {v_{{y_{{\rm {e}}}}}} \enspace {v_{{z_{{\rm {e}}}}}}}\big]^{{\rm {T}}}}}$ is the velocity under EFCS;
{$^{{\rm {e}}}{\bf {a}} = {{\bf {g}}}+\frac{1}{m}\left({}^{{\rm {e}}}{\bf {f}} + {}^{{\rm {e}}}{\bf {f}}_{\rm{drag}}\right)$ is the accelerate under EFCS;}
$m$ is the mass of the multicopter; 
${\bf {g}}$ is the local gravitational acceleration, ${\bf {g}} = \big[{ {0} \enspace {0} \enspace {g}}\big]^{\rm{T}}$;
${}^{{\rm {e}}}{\bf {f}}\in\mathbb{R}^{3}$ is the \textit{controllable force} of the multicopter in EFCS; 
{${}^{{\rm {e}}}{\bf {f}}_{\rm{drag}} = -{\bf{R}}_{\rm{b}}^{\rm{e}}{\bf{D}}{{\bf{R}}_{\rm{b}}^{\rm{e}}}^{\rm{T}} {}^{{\rm {e}}}{\bf {v}}$ is the aerodynamic drag force of the multicopter in EFCS, ${\bf{D}} = {\rm{diag}}\left(d_x, d_y, d_z\right) $ is a constant diagonal matrix that defines the rotor-drag coefficients; }
${^{{\rm {b}}}{\bf {\omega}}={\big[{{\omega_{{x_{{\rm {b}}}}}} \enspace {\omega_{{y_{{\rm {b}}}}}} \enspace {\omega_{{z_{{\rm {b}}}}}}}\big]^{{\rm {T}}}}}$ is the angular velocity under BCS;
$\mathbf{J}$ is the multicopter moments of inertia; 
${{{\bf {G}}_{{\rm {a}}}}}$ is gyroscopic torque; 
$\mathbf{\tau}$ is the moments generated by the propellers in the body axes.
Matrix ${{\left[{^{{\rm {b}}}{\bf {\omega}}}\right]_{\times}}}$
is a skew-symmetric matrix, which is a cross product map of the angular
velocity ${^{{\rm {b}}}{\bf {\omega}}}$. 

The cross product map ${{\left[\cdot\right]_{\times}}}$:
${\mathbb{R}^{3}}\to\mathfrak{so}\left(3\right)$ is defined such
that ${{\left[{\bf {x}}\right]_{\times}}{\bf {y}}={\bf {x}}\times{\bf {y}}}$
and ${\left[{\bf {x}}\right]_{\times}^{{\rm {T}}}=-{\left[{\bf {x}}\right]_{\times}}}$
for any ${{\bf {x}},{\bf {y}}\in\mathbb{R}{^{3}}}$. The inverse of the cross product map is denoted by the vex map ${{\mathop{{\rm vex}}\nolimits}\left(\cdot\right)}$:
$\mathfrak{so}\left(3\right)\to{\mathbb{R}^{3}}$, where the set of 3 by
3 skew-symmetric matrices are denoted by $\mathfrak{so}\left(3\right)=\left\{ {{\bf {S}}\in\mathbb{R}{^{3\times3}}|{\bf {S}}=-{{\bf {S}}^{{\rm {T}}}}}\right\} $.
For ${\bf {x}}\triangleq{\big[{{x_{1}} \enspace {x_{2}} \enspace {x_{3}}}\big]^{{\rm {T}}}}\in\mathbb{R}{^{3}}$, the cross product map and vex map are summarized as
\[
{\left[{\bf {x}}\right]_{\times}}\triangleq\left[{\begin{array}{ccc}
0 & {-{x_{3}}} & {x_{2}}\\
{x_{3}} & 0 & {-{x_{1}}}\\
{-{x_{2}}} & {x_{1}} & 0
\end{array}}\right],{\rm {vex}}\left({{\left[{\bf {x}}\right]}_{\times}}\right)\triangleq{\bf {x}}.
\]

For multicopters, the controllable force is generated by the thrust of the propellers as
{\begin{equation}
	{}^{{\rm {e}}}{\bf {f}}=f{\bf {R}}_{{\rm {b}}}^{{\rm {e}}}{\bf {R}}_{{\rm {f}}}{{\bf {e}}_{3}}
	= f{\bf{n}}_{\rm{f}} \label{eq:f-multicopter}
\end{equation}}
where $0 \le f \le {f_{{\rm {m}}}}$, ${f_{{\rm {m}}}}$ is the maximum thrust;
${\bf {R}}_{{\rm {f}}}\in{\mathsf{SO}}\left ( 3 \right )$ is the rotation matrix transforming the thrust vector
to $o_{\rm{b}}z_{\rm{b}}$ axis in BCS;
${\bf {R}}_{{\rm {f}}}$ is constant for multicopters because the propellers provide thrust in a fixed direction relative to the fuselage, and variable for tiltrotor aircrafts. 
Vector ${{\bf {e}}_{3}}$ is the third column of the $3\times 3$ identity matrix ${\bf{I}}$, ${\bf{I}} = \big[{ {{\bf {e}}_{1}} \enspace {{\bf {e}}_{2}} \enspace {{\bf {e}}_{3}} }\big]$,
${{\bf{e}}_{3}} = \big[{ {0} \enspace {0} \enspace {1} }\big]^{\rm{T}}$.
{Vector ${\bf{n}}_{\rm{f}}$ is the unit vector of ${}^{{\rm {e}}}{\bf {f}}$.}

\subsubsection{Target Model}
{The target model's motion is represented using a point-mass model as
\begin{equation}
	\left\{ \begin{aligned}
		^{{\rm {e}}}\dot{{\bf {p}}}_{\rm{t}} & = {^{{\rm {e}}}{\bf {v}}}_{\rm{t}}\\
		^{{\rm {e}}}\dot{{\bf {v}}}_{\rm{t}} & = {^{{\rm {e}}}{\bf {a}}}_{\rm{t}}
	\end{aligned}
	\right.\label{eq:point-mass-model}
\end{equation}
where $^{{\rm {e}}}{\bf {p}}_{\rm{t}} \in \mathbb{{R}}^3$, $^{{\rm {e}}}{\bf {v}}_{\rm{t}} \in \mathbb{{R}}^3$ and $^{{\rm {e}}}{\bf {a}}_{\rm{t}} \in \mathbb{{R}}^3$ 
are the target's position, velocity, and acceleration under EFCS.}

{Using ${\bf {p}}_{\rm{r}}$ to indicate the relative position of the interceptor and target in EFCS, ${\bf {p}}_{\rm{r}} = {^{{\rm {e}}}{{\bf {p}}}} - {^{{\rm {e}}}{{\bf {p}}_{{\rm {t}}}}}$.
Further, the relative velocity is defined as ${\bf {v}}_{\rm{r}} = {^{{\rm {e}}}{\bf {v}} - {}^{{\rm {e}}}{{\bf {v}}_{{\rm {t}}}}}$, $ \dot{{\bf {p}}}_{\rm{r}} = {\bf {v}}_{\rm{r}}$; 
the relative acceleration is defined as ${\bf {a}}_{\rm{r}} = {^{{\rm {e}}}{\bf {a}} - {}^{{\rm {e}}}{{\bf {a}}_{{\rm {t}}}}}$, $ \dot{{\bf {v}}}_{\rm{r}} = {\bf {a}}_{\rm{r}}$.
Note that measuring the target's motion acceleration ${}^{{\rm {e}}}{{\bf {a}}_{{\rm {t}}}}$ using visual sensors is challenging. 
Therefore, in the design of both the controller and the state observer, ${}^{{\rm {e}}}{{\bf {a}}_{{\rm {t}}}}$ is considered as a disturbance.}

\subsubsection{Camera Imaging Model}
As shown in Fig. \ref{fig:Description-of-the-interception-problem}, the perspective projection performs a mapping from the 3-D space to the 2-D image plane: $\mathbb{R}{^{3}} \to \mathbb{R}{^{2}}$.

\begin{assumption}
	\label{assumption:origin-coincides}
	The origin of CCS coincides with the origin of BCS. That is, their translation
	vector ${{\bf {t}}_{{\rm {c}}}^{{\rm {b}}}={\bf {0}}}$.
	{Due to the strapdown camera configuration, the rotation matrix ${\bf {R}}_{{\rm {c}}}^{{\rm {b}}}$ is constant.}
\end{assumption}

The camera used in this paper is fixed to the body, and it is considered a successful interception when the camera is in contact with the target. 
Thus, the optical axis has a constant rotation ${\bf {R}}_{{\rm {c}}}^{{\rm {b}}}$ with the multicopter head direction.

The LOS is a line between the interceptor and the target. 
The unit vector along it from the interceptor to the target is called the \textit{LOS vector}, denoted as ${{{\bf {n}}_{{\rm {t}}}}}$.
When the target is in the sensor's Field of View (FOV), all types of sensors above can be transformed into the
perspective view through a unified imaging model \cite[Ch. 11]{corke2011robotics}.
{Vector ${{{\bf {n}}_{{\rm {t}}}}}$ is expressed as the unit vector in the negative direction of ${\bf {p}}_{\rm{r}}$, or obtained by image coordinates and focal of length $f_{{\rm {oc}}}$ as}

{\begin{equation}
	\begin{aligned}{{\bf {n}}_{{\rm {t}}}} & =-\frac{{\bf {p}}_{\rm{r}}}{{\left\Vert {\bf {p}}_{\rm{r}}\right\Vert }}\\
		& ={\bf {R}}_{{\rm {b}}}^{{\rm {e}}}{\bf {R}}_{{\rm {c}}}^{{\rm {b}}}{\frac{{\left[{	{p_{{x_{{\rm {i}}}}}} \enspace {p_{{y_{{\rm {i}}}}}} \enspace {f_{{\rm {oc}}}} }\right]}^{{\rm {T}}}}{{\left\Vert {\left[{	{p_{{x_{{\rm {i}}}}}} \enspace {p_{{y_{{\rm {i}}}}}} \enspace {f_{{\rm {oc}}}} }\right]}^{{\rm {T}}}\right\Vert }}}.
	\end{aligned}
	\label{eq:measurement model}
\end{equation}}

The FOV is the range of the observable area that a perception sensor can image, including the Horizontal Field of View (HFOV) and Vertical Field of View (VFOV). 
The parameters ${{\alpha_{{\rm {hfov}}}}}$, ${{\alpha_{{\rm {vfov}}}}}$ are the angles of HFOV and VFOV of the perception sensor. 
The FOV is defined as 
$ {\mathcal C} = \left\{ {{\bf {x}} \mid  \left | \arctan \frac{x_{\rm{c}1}}{x_{\rm{c}3}} \right | \le \frac{1}{2} \alpha_{\rm{hfov}}, \left | \arctan \frac{x_{\rm{c}2}}{x_{\rm{c}3}} \right | \le \frac{1}{2} \alpha_{\rm{vfov}}  } \right\} $,
where ${}^{\rm{c}}{\bf {x}}\triangleq{\big[{ {x_{\rm{c}1}} \enspace {x_{\rm{c}2}} \enspace {x_{\rm{c}3}} }\big]^{{\rm {T}}}}$ is ${\bf{x}}$ in CCS.

The camera imaging is subject to the FOV constraint that the target is within the sensor's FOV,
\begin{equation}
	{{\bf {n}}_{{\rm {t}}}} \in {\mathcal C}.
	\label{eq:FOV}
\end{equation}

{The designed LOS vector \({\bf{n}}_{{\rm{td}}}\) can be simply conceptualized as the optical axis vector \({\bf{n}}_{{\rm{c}}}\). In pursuit of enhancing the interceptor's maneuverability, \({\bf{n}}_{{\rm{td}}}\) is crafted to approximate the force vector \({\bf{n}}_{{\rm{f}}}\) as closely as possible within the FOV. Upon achieving the collinear control objective, a greater component of acceleration is directed along \({\bf{n}}_{{\rm{td}}}\), thereby endowing the interceptor with increased interception velocity.}



\subsubsection{IBVS Model}

The image Jacobian matrix \cite{chaumette2006visual} can represent
the relationship between camera motion and image features motion as
\begin{equation}
\setlength{\arraycolsep}{3.8pt}\begin{aligned}^{{\rm {i}}}{\dot {\bar{\bf{p}}}}\! & =\!{{\bf {L}}_{s}}{}^{{\rm {c}}}{\bf {\tilde{v}}},\\
{{\bf {L}}_{s}}\! & =\!\left[{\begin{array}{cccccc}
{-\frac{1}{{p_{{z_{{\rm {c}}}}}}}} & 0 & {\frac{{{\bar{{p}}}_{{x_{{\rm {i}}}}}}}{{p_{{z_{{\rm {c}}}}}}}} & {{{\bar{p}}_{{x_{{\rm {i}}}}}}{{\bar{p}}_{{y_{{\rm {i}}}}}}} & {-\!\left({1\!+\!\bar{p}_{{x_{{\rm {i}}}}}^{2}}\right)} & {{\bar{p}}_{{y_{{\rm {i}}}}}}\\
0 & {-\frac{1}{{p_{{z_{{\rm {c}}}}}}}} & {\frac{{{\bar{p}}_{{y_{{\rm {i}}}}}}}{{p_{{z_{{\rm {c}}}}}}}} & {1\!+\!\bar{p}_{{y_{{\rm {i}}}}}^{2}} & {-\!{{\bar{p}}_{{x_{{\rm {i}}}}}}{{\bar{p}}_{{y_{{\rm {i}}}}}}} & {-\!{{\bar{p}}_{{x_{{\rm {i}}}}}}}
\end{array}}\right]
\end{aligned}
\label{eq:IBVS}
\end{equation}
where $^{{\rm {c}}}{\bf {\tilde{v}}}\!=\!{{\left[{\begin{array}{cc}
^{{\rm {c}}}{\bf {v}}^{{\rm {T}}} & ^{{\rm {c}}}{\bf {\omega}}^{{\rm {T}}}\end{array}}\right]}^{{\rm {T}}}}$ represents the twist in CCS, including the
linear and angular velocity of a rigid body; $^{{\rm {i}}}{\bf {\bar{p}}}={\left[{\begin{array}{cc}
{{\bar{p}}_{{x_{{\rm {i}}}}}} & {{\bar{p}}_{{y_{{\rm {i}}}}}}\end{array}}\right]^{{\rm {T}}}} = {\left[{\begin{array}{cc} {p_{{x_{{\rm {i}}}}}}/{f_{{\rm {oc}}}} & {p_{{y_{{\rm {i}}}}}}/{f_{{\rm {oc}}}}\end{array}}\right]^{{\rm {T}}}}$ represents the normalized image coordinate of features in ICS.

\subsection{State Equation and Measurement Equations}

In addition to the multicopter rigid model and the IBVS model, the description of the multicopter interception process requires knowledge of the measurement equations, including IMU measurement and delayed image measurement.

\subsubsection{System Model}

Logarithmic mapping $\ln\left(\cdot\right):{\mathop{{\mathsf{SO}}}\nolimits}\left(3\right)\to\mathfrak{so}\left(3\right)$
converts rotation matrix ${{\bf {R}}_{{\rm {b}}}^{{\rm {e}}}}$
into a rotation angle $\theta\in\left[-\pi,\pi\right)$ and a unit
rotation vector ${{\bf {a}}}\in\mathbb{{R}}{}^{3}$, that
is
\[
{\left[{\theta{\bf {a}}}\right]_{\times}}=\ln\left({{\bf {R}}_{{\rm {b}}}^{{\rm {e}}}}\right)=\sum\limits _{i=0}^{\infty}{\frac{{{\left({-1}\right)}^{i}}}{{i+1}}{{\left({{\bf {R}}_{{\rm {b}}}^{{\rm {e}}}-{\bf {I}}}\right)}^{i+1}}}.
\]

A 4-D vector is used to compactly represent interceptor's rotation matrix ${{\bf {R}}_{{\rm {b}}}^{{\rm {e}}}}$, called orientation quaternion ${{{\bf {q}}={\big[{\begin{array}{cc}
{\cos\frac{\theta}{2}} & {{{\bf {n}}^{{\rm {T}}}}\sin\frac{\theta}{2}}\end{array}}\big]^{{\rm {T}}}}\in\mathbb{{R}}{^{4}}}}$.
The normalized image coordinate of the center of the target is considered the feature point and is denoted as ${^{{\rm {i}}}{\bf {\bar{p}}}}$.
The state variable ${\bf {x}}={\big[{\begin{array}{cccccc}
{{\bf {q}}^{{\rm {T}}}} & {{\bf {p}}_{\rm{r}}^{{\rm {T}}}} & {{\bf {v}}_{\rm{r}}^{{\rm {T}}}} & {^{{\rm {i}}}{{\bf {\bar{p}}}^{{\rm {T}}}}} & {{\bf {b}}_{{\rm {gyr}}}^{{\rm {T}}}} & {{\bf {b}}_{{\rm {acc}}}^{{\rm {T}}}}\end{array}}\big]^{{\rm {T}}}}\in\mathbb{{R}}{^{18}}$ consists of some necessary states, including interceptor's orientation quaternion ${{\bf {q}}}$, relative position ${\bf {p}}_{\rm{r}}$
and relative velocity ${{\bf {v}}}_{\rm{r}}$,
image feature point ${^{{\rm {i}}}{\bf {\bar{p}}}}$, gyroscope bias ${{\bf {b}}_{{\rm {gyr}}}}={\big[{\begin{array}{ccc}
{b_{{\rm {gyr}},x}} & {b_{{\rm {gyr}},y}} & {b_{{\rm {gyr}},z}}\end{array}}\big]^{{\rm {T}}}}$, and accelerometer bias ${{{\bf {b}}_{{\rm {acc}}}}={\big[{\begin{array}{ccc}
{b_{{\rm {acc}},x}} & {b_{{\rm {acc}},y}} & {b_{{\rm {acc}},z}}\end{array}}\big]^{{\rm {T}}}}}$.

The discrete state equation is written as
\[
{\bf {x}}_{k}={\bf {F}}_{k}{\bf {x}}_{k-1}+{\bf {G}}_{k}{\bf {w}}_{k-1}
\]
where ${\bf {F}}_{k}\in\mathbb{{R}}{}^{18\times18}$ and ${\bf {G}}_{k}\in\mathbb{{R}}{}^{18\times6}$
are the state transition matrix and the noise matrix from time ${{t_{k-1}}}$
to time ${{t_{k}}}$; ${\bf {w}}_{k}$ is the process noise
vector, where the elements follow the Gaussian distribution with zero
mean and uncorrelated, ${\bf {w}}_{k}\sim{\cal N}\left({{{\bf {0}}_{6\times1}},{{\bf {Q}}_{k}}}\right),{{\bf {Q}}_{k}}\in\mathbb{{R}}{}^{6\times6}$.
Vector ${\bf {w}}_{k}$ consists of gyroscope bias noise ${{\bf {n}}_{{{\bf {b}}_{{\rm {gyr}}}}}}$ and accelerometer bias noise ${{\bf {n}}_{{{\bf {b}}_{{\rm {acc}}}}}}$, ${\bf {w}}_{k} = \big[{\begin{array}{cc}{{\bf {n}}_{{{\bf {b}}_{{\rm {gyr}}}}}^{{\rm {T}}}} & {{\bf {n}}_{{{\bf {b}}_{{\rm {acc}}}}}^{{\rm {T}}}}\end{array}}\big]^{{\rm {T}}}$.

\subsubsection{IMU Measurements}

IMU provides gyroscope angle increment and accelerometer velocity increment under BCS. 
Combining these measurements with the multicopter rigid model \eqref{eq:rigid-body-model}, the recursive expressions of states ${{\bf {q}}}$ and ${{\bf {v}}}_{\rm{r}}$ are obtained.
\begin{enumerate}[(i)]
\item The derivative of the orientation quaternion ${{\bf {q}}}$,
which is the angular velocity $^{{\rm {b}}}{\bf {\omega}}$, is obtained
from the gyroscope measurement as
\begin{equation}
^{{\rm {b}}}{\bf {\omega}}={{\bf {\omega}}_{{\rm {gyr}}}}-{{\bf {b}}_{{\rm {gyr}}}}-{{\bf {n}}_{{\rm {gyr}}}}\label{eq:gyr-measurement}
\end{equation}
where the gyroscope measurement noise ${{\bf {n}}_{{\rm {gyr}}}}$ is a zero-mean Gaussian random variable; the bias ${{{\bf {b}}_{{\rm {gyr}}}}}$ obeys the Wiener process,
\begin{equation}
{{{\dot {\bf{b}}}_{{\rm {gyr}}}}={{\bf {n}}_{{{\bf {b}}_{{\rm {gyr}}}}}}}.\label{eq:dotb_gyr}
\end{equation}

\item The derivative of the velocity ${{\bf {v}}}_{\rm{r}}$, which is
the acceleration $^{{\rm {e}}}{\bf {a}}$, is obtained from the measured
value of the accelerometer as
\begin{equation}
^{{\rm {e}}}{\bf {a}}=g{{\bf {e}}_{3}}+{\bf {R}}_{{\rm {b}}}^{{\rm {e}}}\left({{{\bf {a}}_{{\rm {acc}}}}-{{\bf {b}}_{{\rm {acc}}}}-{{\bf {n}}_{{\rm {acc}}}}}\right)\label{eq:ve-and-higher-order-state}
\end{equation}
where the accelerometer measurement noise ${{\bf {n}}_{{\rm {acc}}}}$ is a zero-mean Gaussian random variable; the bias ${{{\bf {b}}_{{\rm {acc}}}}}$ obeys the Wiener process,
\begin{equation}
{{{\dot {\bf{b}}}_{{\rm {acc}}}}={{\bf {n}}_{{{\bf {b}}_{{\rm {acc}}}}}}}.\label{eq:dotb_acc}
\end{equation}
\end{enumerate}

\subsubsection{Delayed Image Measurements}

Image processing provides the image coordinates of the target point,
and the state ${^{{\rm {i}}}{\bf {\bar{p}}}}$ is obtained
after normalization. Since camera imaging and image processing
take time, the measured value obtained is data from a period ago.
Suppose that the measured delay is $D$ cycles based on the IMU update frequency, and the delay time ${{t_{D}}>0}$ is known. Then the image measurement at time $t$ is expressed as
\begin{equation}
{^{{\rm {i}}}{{\bf {\bar{p}}}_{{\rm {img}},k}}={}^{{\rm {i}}}{\bf {\bar{p}}}_{k-D}+{{\bf {n}}_{{\rm {img}}}}}\label{eq:p^bar_img,k}
\end{equation}
where the image measurement noise ${{\bf {n}}_{{\rm {img}}}}$ is a zero-mean Gaussian random variable.
{Note that the delay time ${t_{D}}$ involves camera imaging and image processing time and was estimated using two methods: (i) direct calculation based on camera parameters and internal program timers, and (ii) experimental measurement involving capturing a timer with the camera.}

\subsection{Controller Design and State Observer Problem Formulation}


Based on \textit{Assumption \ref{assumption:origin-coincides}}, the controller design and state observer problems are summarized as follows.

\begin{problem}
	\textit{(Controller Design Problem)}
	\label{problem:Controller-design}
	Under \textit{Assumption \ref{assumption:origin-coincides}}, it is assumed that the state information, including the multicopter and the image measurement, is accurately available. At the same time, the positions of target and interceptor in the inertial system are unknown. 
	{For the interceptor model \eqref{eq:rigid-body-model}, design controllers for $f$, $^{{\rm {b}}}{\bf {\omega}}$ such that the target unit vector remains proximate to the designed LOS vector, i.e., $1-{\bf{n}}_{{\rm{td}}}^{{\rm {T}}}{{\bf {n}}_{{\rm {t}}}} < \epsilon$, where $\epsilon>0$ is a range threshold, and the interceptor multicopter further approaches the target, i.e., ${\bf {p}}_{\rm{r}} \to {\bf{0}}$.}
\end{problem}

\begin{problem}
	\textit{(State Observer Problem)}
	\label{problem:State-observer}
	Given the initial value ${{\bf {x}}\left(0\right)}$, model error and measurement error distribution, how to use IMU measurement and delayed image measurement to calculate more accurate status information.
\end{problem}

In \textit{Problem \ref{problem:Controller-design}}, a controller is designed to perform the interception task with accurate image measurement and IMU information feedback. However, due to the time required for camera imaging and processing, around \SI{80}{ms} for color recognition in our experiments, significant errors can occur in high-speed flights. To mitigate control quality degradation from such delays, image estimation methods are essential. The challenge of choosing a suitable state vector $\mathbf{x}$ and estimating the position of image features in real-time with a filter is addressed in \textit{Problem \ref{problem:State-observer}}.

\section{High-speed Interception Control based on Rotation Matrix\label{sec:high-speed-interception-control}}

In this section, \textit{Problem \ref{problem:Controller-design}} is solved by designing the angular velocity controller so that the designed LOS vector ${{{\bf {n}}_{{\rm {td}}}}}$, the velocity vector ${\bf {v}}_{\rm {r}}$ and the target unit vector ${{{\bf {n}}_{{\rm {t}}}}}$ are collinear.
{The detailed derivations and proof of the theorem in this section are placed in \textit{Appendix \ref{subsec:appendix-proof}}.}

\subsection{High-speed interception control process}

\begin{figure}
	\centering{}\includegraphics[scale=0.9]{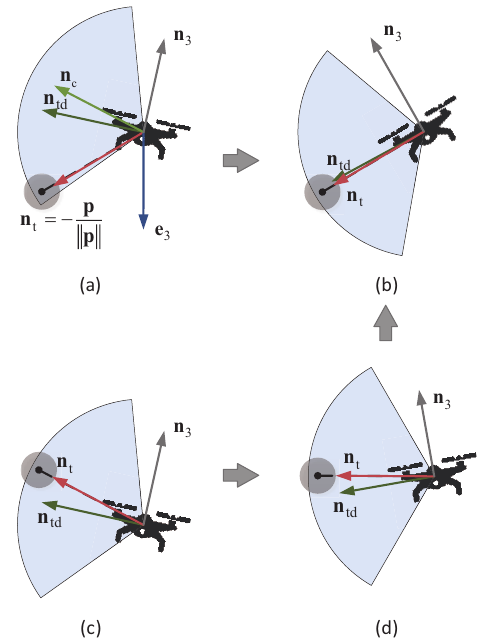}\caption{\label{fig:High-speed-interception-control}High-speed interception
	control process. The process for intercepting a lower target is (a)--(b),
	and the process for intercepting a higher target is (c)--(d)--(b).}
	\vspace{-0.3cm}
\end{figure}

Figures \ref{fig:High-speed-interception-control}(a) to \ref{fig:High-speed-interception-control}(c) illustrate the design of the attitude rate loop controller, aiming to align the designed LOS vector ${\bf{n}}_{{\rm{td}}}$ and the target unit vector ${{{\bf {n}}_{{\rm {t}}}}}$ for effective target locking and interception. However, in initial states where the target is above the interceptor, as shown in Fig. \ref{fig:High-speed-interception-control}(c), this alignment alone can cause backward acceleration. Given the multicopter's underactuated nature, where acceleration and attitude are coupled, it's crucial to also ensure the collinearity of the velocity vector ${\bf {v}}_{\rm {r}}$ and the target unit vector ${{{\bf {n}}_{{\rm {t}}}}}$. This approach allows for the calculation of desired force ${\bf{f}}_{{\rm{d}}}$ and rotation matrix ${\bf{R}}_{{\rm{d}}}$ in achieving the desired velocity.

The controller follows two main objectives: ensuring the collinearity of the LOS vector and the target unit vector, and aligning the velocity vector with the target unit vector. The high-speed interception controller comprises two parts: a collinear controller (Section \ref{subsec:Control-Design-for-nh-nt}) as the outer loop running at \SI{50}{\hertz}, and an attitude controller (Section \ref{subsec:Control-Design-for-ve-nt}) as the inner loop running at \SI{200}{\hertz}. The collinear controller ensures the target remains within FOV, while the attitude controller guarantees that the multicopter's velocity and attitude accurately track the reference values from the outer loop controller. This controller is outlined in Fig. \ref{fig:Control scheme} and \textit{Algorithm \ref{alg:Controller}}.

\begin{figure}
	\centering{}\includegraphics{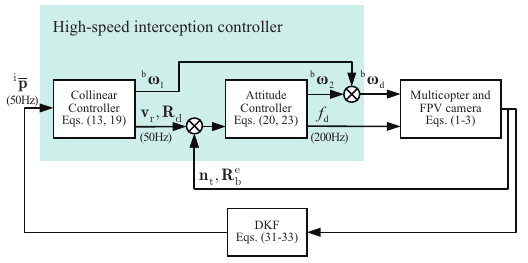}
	
	\caption{\label{fig:Control scheme}Entire control scheme. The high-speed interception
	controller is divided into a collinear controller in the outer
	loop and an attitude controller in the inner loop.}
\end{figure}

\begin{algorithm}[t]
	\caption{\label{alg:Controller}{$f_{\rm{d}}$, ${\bf {\omega}}_{\rm{d}}$ = InterceptionController($^{{\rm {i}}}{\bf {\bar{p}}}$,
	${\bf {v}}_{\rm {r}}$, ${\bf {R}}_{{\rm {b}}}^{{\rm {e}}}$)}}
	{
	\begin{algorithmic}[1] 
	\Require{Image measurement, relative velocity and self attitude} 
	\Ensure{Interception controller commands} 
	\State{Initialize ${{\bf {a}}}_{\rm{d}} = \bf{0}$, ${\bf {R}}_{{\rm {d}}} = {\bf {R}}_{{\rm {b}}}^{{\rm {e}}}$}
	\While{Interception in progress}
		\State{Calculate ${\bf {n}}_{{\rm {t}}}$ from $^{{\rm {i}}}{\bf {\bar{p}}}$ and Eq. \eqref{eq:measurement model}}
		\State{Calculate the attitude controller commands $f_{\rm{d}}$, ${\bf {\omega}}_{\rm{d}}$ according to Eqs. \eqref{eq:fd}, \eqref{eq:w2}}
		\If{${}^{\rm{i}}{{\bf{\bar p}}}$ comes}
			\State{Calculate the collinear controller commands ${\bf {\omega}}_{1}$ and attitude controller desired value ${{\bf {a}}_{{\rm {d}}}}$, ${\bf {R}}_{{\rm {d}}}$ according to Eqs. \eqref{eq:w1}, \eqref{eq:ad}, \eqref{eq:Rd}}
			\State{${\bf {\omega}}_{\rm{d}} = {\bf {\omega}}_{\rm{d}} + {\bf {\omega}}_{1}$}
		\EndIf
		\State{wait($T$). $T$ = iteration time - this loop time}
	\EndWhile
	\end{algorithmic}}
\end{algorithm}

\subsection{Collinear Control Design for designed LOS vector and Target Unit
Vector\label{subsec:Control-Design-for-nh-nt}}
{
Let LOS tracking error as $z_1 = 1-{\bf{n}}_{{\rm{td}}}^{{\rm {T}}}{{\bf {n}}_{{\rm {t}}}}$ and take the derivative of it, we have
\[
\begin{aligned}{{\dot{z}}_{1}} & =-\left({{\bf {n}}_{{\rm {td}}}^{{\rm {T}}}{{\dot {\bf{n}}}_{{\rm {t}}}}+{\bf {n}}_{{\rm {t}}}^{{\rm {T}}}{{\dot {\bf{n}}}_{{\rm {td}}}}}\right)\\
 & =-\frac{1}{{\left\Vert {\bf {p}}_{\rm{r}}\right\Vert }}{\bf {n}}_{{\rm {td}}}^{{\rm {T}}}\left({-{\bf {I}}+{{\bf {n}}_{{\rm {t}}}}{\bf {n}}_{{\rm {t}}}^{{\rm {T}}}}\right){{\bf {v}}_{\rm{r}}}-{\left( {{{\bf{n}}_{{\rm{td}}}} \!\times\! {{\bf{n}}_{\rm{t}}}} \right)^{\rm{T}}}{}^{\rm{e}}\omega.
\end{aligned}
\]
Choose the following barrier Lyapunov candidate function,
\begin{equation}
	{L_1} = \frac{1}{2}\log \frac{{k_{\rm{b}}^2}}{{k_{\rm{b}}^2 - z_1^2}}\label{eq:L1}
\end{equation}
where $\log\left(\cdot\right)$ denoted the natural logarithm, and $k_{\rm{b}}$ constant on $z_1$, that is, we require $\left\lvert z_1\right\rvert < k_{\rm{b}}$.
The derivative of Eq. \eqref{eq:L1} along Eq. \eqref{eq:no_dot} is
\[
	\begin{small}
		\begin{aligned}{{\dot{L}}_{1}} & =\frac{z_1 \dot{z}_1}{{k_{\rm{b}}^2 - z_1^2}}\\
			& =- \frac{{{z_1}}}{{k_{\rm{b}}^2 \!-\! z_1^2}}\frac{1}{{\left\| {{{\bf{p}}_{\rm{r}}}} \right\|}}{\bf{n}}_{{\rm{td}}}^{\rm{T}}\left( { - {\bf{I}} \!+\! {{\bf{n}}_{\rm{t}}}{\bf{n}}_{\rm{t}}^{\rm{T}}} \right){{\bf{v}}_{\rm{r}}} \!-\! \frac{{{z_1}}}{{k_{\rm{b}}^2 - z_1^2}}{\left( {{{\bf{n}}_{{\rm{td}}}} \!\times\! {{\bf{n}}_{\rm{t}}}} \right)^{\rm{T}}}{}^{\rm{e}}\omega.
		   \end{aligned}
	\end{small}
\]
The first term is equal to 0 when ${\bf {v}}_{\rm {r}}$ and ${{{\bf {n}}_{{\rm {t}}}}}$ are collinear, which will be designed in the following subsection. 
The designed angular velocity makes the second term negative. 
For control purposes, designed LOS vector ${\bf{n}}_{{\rm{td}}}$ and target unit vector ${{{\bf {n}}_{{\rm {t}}}}}$ are collinear. Then
\begin{equation}
\begin{aligned}^{{\rm {e}}}{\bf {\omega}} & =\frac{{{z_1}}}{{k_{\rm{b}}^2 \!-\! z_1^2}}\left({{\bf{n}}_{{\rm{td}}}\times{{\bf {n}}_{{\rm {t}}}}}\right)\\
\Rightarrow{}^{{\rm {b}}}{{\bf {\omega}}_{1}} & =\frac{{{z_1}}}{{k_{\rm{b}}^2 \!-\! z_1^2}}{\bf {R}}{_{{\rm {b}}}^{{\rm {e}}}}^{{\rm {T}}}\left({{\bf{n}}_{{\rm{td}}}\times{{\bf {n}}_{{\rm {t}}}}}\right).
\end{aligned}
\label{eq:w1}
\end{equation}
According to the physical meaning of the vector cross product, this control command represents the rotational angular velocity from ${\bf{n}}_{{\rm{td}}}$ to ${{{\bf {n}}_{{\rm {t}}}}}$. The term $\frac{{{z_1}}}{{k_{\rm{b}}^2 \!-\! z_1^2}}$ acts as a proportional coefficient, which becomes larger as the LOS tracking error $z_1$ approaches the barrier $k_{\rm{b}}$.}

\subsection{Collinear Control Design for Velocity Vector and Target Unit Vector\label{subsec:Control-Design-for-ve-nt}}
{
The Lyapunov candidate function is designed as
\begin{equation}
	{L_{2}} = \frac{1}{2} {\bf {p}}_{\rm{r}}^{\rm{T}} {\bf {p}}_{\rm{r}} \ge 0.\label{eq:L2}
\end{equation}
The derivative of Eq. \eqref{eq:L2} is
\begin{equation}
	{{\dot{L}}_{2}} = {\bf {p}}_{\rm{r}}^{\rm{T}} {\bf {v}}_{\rm{r}}.\label{eq:L2_dot}
\end{equation}
According to Lyapunov's second method, in order to make the system stable, ${{\dot{L}}_{2}}$ is expected to be negative definite.
Let ${\bf {v}}_{\rm{rd}} = -k_1 {\bf {p}}_{\rm{r}} = -k_1 \left\lVert {\bf {p}}_{\rm{r}}\right\rVert {\bf {n}}_{\rm {t}}$ with $k_1$ being a positive constant, which means that desired relative velocity vector ${{\bf {v}}_{{\rm {r}}}}$ is collinear with target unit vector ${{{\bf {n}}_{{\rm {t}}}}}$. 
Then, rewrite \eqref{eq:L2_dot} as:
\begin{equation}
	\begin{aligned}
		{{\dot{L}}_{2}} 
		&= {\bf {p}}_{\rm{r}}^{\rm{T}} \left({\bf {v}}_{\rm{r}} - {\bf {v}}_{\rm{rd}} + {\bf {v}}_{\rm{rd}}\right)  \\
		&= {\bf {p}}_{\rm{r}}^{\rm{T}} {\bf {z}}_2 - k_1 {\bf {p}}_{\rm{r}}^{\rm{T}} {\bf {p}}_{\rm{r}}.
		\label{eq:L2_dot_rewritten}
	\end{aligned}
\end{equation}
Let a tracking error as ${\bf {z}}_2 = {\bf {v}}_{\rm{r}} - {\bf {v}}_{\rm{rd}} = {\bf {v}}_{\rm{r}} + k_1 {\bf {p}}_{\rm{r}}$ and take the derivative of it, we have
\begin{equation*}
	{\dot{\bf {z}}}_2 = {\bf {a}}_{\rm{r}} + k_1 {\bf {v}}_{\rm{r}}.
\end{equation*}
}

{
Define a new Lyapunov candidate function as:
\begin{equation}
	{L_{3}} = {L_{1}} + {L_{2}} + \frac{1}{2} {\bf {z}}_{2}^{\rm{T}} {\bf {z}}_{2} \ge 0.\label{eq:L3}
\end{equation}
The derivative of Eq. \eqref{eq:L3} is
\begin{equation}
	\begin{small}
		\begin{aligned}
			{{\dot{L}}_{3}} \!=\! &- \frac{{{z_1}}}{{k_{\rm{b}}^2 \!-\! z_1^2}}\frac{1}{{\left\| {{{\bf{p}}_{\rm{r}}}} \right\|}}{\bf{n}}_{{\rm{td}}}^{\rm{T}}\left( { - {\bf{I}} \!+\! {{\bf{n}}_{\rm{t}}}{\bf{n}}_{\rm{t}}^{\rm{T}}} \right){\bf {z}}_2 - \frac{{{z_1}}}{{k_{\rm{b}}^2 \!-\! z_1^2}}{\left( {{{\bf{n}}_{{\rm{td}}}} \!\!\times\!\! {{\bf{n}}_{\rm{t}}}} \right)^{\rm{T}}}{}^{\rm{e}}\omega \\
			 &+ {\bf {p}}_{\rm{r}}^{\rm{T}} {\bf {z}}_2 \!-\! k_1 {\bf {p}}_{\rm{r}}^{\rm{T}} {\bf {p}}_{\rm{r}} \!+\! {\bf {z}}_{2}^{\rm{T}} \left({\bf {a}}_{\rm{r}} \!+\! k_1 {\bf {v}}_{\rm{r}} \right). 
		\end{aligned}
	\end{small}
\end{equation}}

{
In order for the system to be stable, ${{\dot{L}}_{3}}$ is expected to be negative definite. 
Thus, let ${{{\bf {a}}_{{\rm {r}}}}} = -k_1 {\bf {v}}_{\rm{r}} - k_2 {\bf {z}}_{2} - {\bf {p}}_{\rm{r}} + \frac{{{z_1}}}{{k_{\rm{b}}^2 - z_1^2}}\frac{m}{{\left\| {{{\bf{p}}_{\rm{r}}}} \right\|}}\left( { - {\bf{I}} + {{\bf{n}}_{\rm{t}}}{\bf{n}}_{\rm{t}}^{\rm{T}}} \right){{\bf{n}}_{{\rm{td}}}}$ and $k_2$ is a positive constant.
The desired acceleration ${{{\bf {a}}_{{\rm {d}}}}}$ is designed as
\begin{equation}
{{\bf {a}}_{{\rm {d}}}} = -k_1 {\bf {v}}_{\rm{r}} - k_2 {\bf {z}}_{2} - {\bf {p}}_{\rm{r}} + \frac{{{z_1}}}{{k_{\rm{b}}^2 \!-\! z_1^2}}\frac{m}{{\left\| {{{\bf{p}}_{\rm{r}}}} \right\|}}\left( { -\! {\bf{I}} \!+\! {{\bf{n}}_{\rm{t}}}{\bf{n}}_{\rm{t}}^{\rm{T}}} \right){{\bf{n}}_{{\rm{td}}}} + {}^{{\rm {e}}}{{\bf {a}}_{{\rm {t}}}}.\label{eq:ad}
\end{equation}}

\subsubsection{Control Design for Desired Acceleration ${{\bf {a}}_{{\rm {d}}}}$}

The desired acceleration ${{{\bf {a}}_{{\rm {d}}}}}$ is designed to realize the difference between the actual velocity and the desired velocity. According to the multicopter dynamics model \eqref{eq:rigid-body-model}, we have
{
\begin{equation}
{{\bf {a}}_{{\rm {d}}}}={{\bf {g}}} + \frac{f}{m}{\bf {n}_{{\rm {fd}}}} + \frac{1}{m}{}^{{\rm {e}}}{\bf {f}}_{\rm{drag}}.\label{eq:ad-model}
\end{equation}
Then,
\[
{\bf {n}_{{\rm {fd}}}}=\frac{{m{{{\bf {a}}_{{\rm {d}}}} - m{{\bf {g}}}}} - {}^{{\rm {e}}}{\bf {f}}_{\rm{drag}}}{{f_{{\rm {d}}}}}.
\]
Because ${\bf {n}_{{\rm {fd}}}}$ satisfies ${{\bf {n}}_{{\rm {fd}}}^{{\rm {T}}}{\bf {n}_{{\rm {fd}}}}=1}$,
the equation above becomes
\begin{equation}
{\bf {n}_{{\rm {fd}}}}=\frac{{{{\bf {a}}_{{\rm {d}}}}-{{\bf {g}}}} - {}^{{\rm {e}}}{\bf {f}}_{\rm{drag}}/m}{{\left\Vert {{{\bf {a}}_{{\rm {d}}}}-{{\bf {g}}}} - {}^{{\rm {e}}}{\bf {f}}_{\rm{drag}}/m\right\Vert }}.\label{eq:nf}
\end{equation}}

Since the yaw motion of the interceptor will cause the target to move laterally on the image, it is better to maintain the current yaw to achieve the desired attitude ${{{\bf {R}}_{{\rm {d}}}}}$.
A rotation that includes only pitch-roll is defined as ${{{\bf {R}}_{{\rm {tilt}}}}}$, and the calculation of ${{{\bf {R}}_{{\rm {tilt}}}}}$ takes the cross product of the current frame and the desired frame \cite{yu2015high}. 
{The cross product of ${\bf {n}_{{\rm {f}}}}$ and ${\bf {n}_{{\rm {fd}}}}$ is the rotation axis ${\bf {r}}$ of ${{{\bf {R}}_{{\rm {tilt}}}}}$, ${\bf{r}} = {{\bf{n}}_{\rm{f}}} \times {{\bf{n}}_{{\rm{fd}}}}$. 
And the angle $\phi$ between these two axes is the rotation angle, $\phi  = \arccos \left( {{\bf{n}}_{\rm{f}}^{\rm{T}}{{\bf{n}}_{{\rm{fd}}}}} \right)$. }
Therefore, desired attitude of interceptor ${{{\bf {R}}_{{\rm {d}}}}}$ is obtained as
{\begin{equation}
	\begin{aligned}
		{\bf {R}}_{{\rm {d}}} &= {{{\bf {R}}_{{\rm {tilt}}}}}{\bf {R}}_{{\rm {b}}}^{{\rm {e}}}, \\
		{{{\bf {R}}_{{\rm {tilt}}}}} &= e^{\phi\left[{\bf {r}}\right]_{\times}} = {\bf {I}}+\left[{\bf {r}}\right]_{\times}\sin\phi+\left[{\bf {r}}\right]_{\times}^{2}\left(1-\cos\phi\right)
	\end{aligned}
\label{eq:Rd}
\end{equation}}
where the exponential map $e^{\left(\cdot\right)}:\mathfrak{so}\left(3\right)\to{\mathsf{SO}}\left(3\right)$ converts the rotation angle $\phi\in\left[-\pi,\pi\right)$ and unit rotation vector ${\bf {r}}\in\mathbb{{R}}{}^{3}$ into a rotation matrix ${{{\bf {R}}_{{\rm {tilt}}}}}$.

Combining Eqs. \eqref{eq:ad} and \eqref{eq:nf}, the desired thrust is calculated as
{
\[
{f_{{\rm {d}}}}={\bf {n}}_{{\rm {f}}}^{{\rm {T}}}\left(m{{{\bf {a}}_{{\rm {d}}}} - m{{\bf {g}}}} - {}^{{\rm {e}}}{\bf {f}}_{\rm{drag}}\right).
\]
In practice, the thrust of a multicopter cannot be negative. Therefore, a saturation limit must be added as
\begin{equation}
	{f_{{\rm {d}}}} =\min\left(\max\left({{\bf {n}}_{{\rm {f}}}^{{\rm {T}}}\left({{\bf {a}}_{{\rm {d}}}} \!-\! m{{\bf {g}}} \!-\! {}^{{\rm {e}}}{\bf {f}}_{\rm{drag}}\right),0}\right),f_{{\rm {m}}}\right).\label{eq:fd}
\end{equation}}

\subsubsection{Control Design for Desired Attitude Matrix ${{{\bf {R}}_{{\rm {d}}}}}$}

The Lyapunov candidate function is designed as
\begin{equation}
{L_{4}}={\mathop{{\rm tr}}\nolimits}\left({{\bf {I}}-{\bf {R}}_{{\rm {d}}}^{{\rm {T}}}{\bf {R}}_{{\rm {b}}}^{{\rm {e}}}}\right).\label{eq:L4}
\end{equation}
When the actual attitude ${{\bf {R}}_{{\rm {b}}}^{{\rm {e}}}}$ of the interceptor is consistent with the desired attitude ${{{\bf {R}}_{{\rm {d}}}}}$, ${\bf {R}}_{{\rm {d}}}^{{\rm {T}}}{\bf {R}}_{{\rm {b}}}^{{\rm {e}}}={\bf {I}}$.
The Frobenius norm is used to define the matrix error,
\[
\begin{aligned}{\left\Vert {{\bf {I}}-{\bf {R}}_{{\rm {d}}}^{{\rm {T}}}{\bf {R}}_{{\rm {b}}}^{{\rm {e}}}}\right\Vert _{F}} & =\sqrt{{\mathop{{\rm tr}}\nolimits}\left({\left({{\bf {I}}-{\bf {R}}_{{\rm {d}}}^{{\rm {T}}}{\bf {R}}_{{\rm {b}}}^{{\rm {e}}}}\right){{\left({{\bf {I}}-{\bf {R}}_{{\rm {d}}}^{{\rm {T}}}{\bf {R}}_{{\rm {b}}}^{{\rm {e}}}}\right)}^{{\rm {T}}}}}\right)}\\
 & =\sqrt{2{\mathop{{\rm tr}}\nolimits}\left({{\bf {I}}-{\bf {R}}_{{\rm {d}}}^{{\rm {T}}}{\bf {R}}_{{\rm {b}}}^{{\rm {e}}}}\right)}.
\end{aligned}
\]
So, ${{L_{4}}=\left\Vert {{\bf {I}}-{\bf {R}}_{{\rm {d}}}^{{\rm {T}}}{\bf {R}}_{{\rm {b}}}^{{\rm {e}}}}\right\Vert _{F}^{2}/2\ge0}$.
The derivative of Eq. \eqref{eq:L4} along Eq. \eqref{eq:rigid-body-model} is \cite{murray2017mathematical,lee2010geometric}
\begin{equation}
\begin{aligned}{{\dot{L}}_{4}} & =-{\mathop{{\rm tr}}\nolimits}\left({{\bf {R}}_{{\rm {d}}}^{{\rm {T}}}{\bf {R}}_{{\rm {b}}}^{{\rm {e}}}{{\left[{^{{\rm {b}}}{\bf {\omega}}}\right]}_{\times}}}\right)\\
 & ={\mathop{{\rm vex}}\nolimits}\left({{\bf {R}}_{{\rm {d}}}^{{\rm {T}}}{\bf {R}}_{{\rm {b}}}^{{\rm {e}}}-{\bf {R}}{{_{{\rm {b}}}^{{\rm {e}}}}^{{\rm {T}}}}{{\bf {R}}_{{\rm {d}}}}}\right)^{\rm {T}} {}^{{\rm {b}}}{\bf {\omega}}.
\end{aligned}
\label{eq:L4_dot}
\end{equation}

The angular velocity controller is designed as
\begin{equation}
^{{\rm {b}}}{{\bf {\omega}}_{2}}=-{\mathop{{\rm vex}}\nolimits}\left({{\bf {R}}_{{\rm {d}}}^{{\rm {T}}}{\bf {R}}_ {{\rm {b}}}^{{\rm {e}}}-{\bf {R}}{{_{{\rm {b}}}^{{\rm {e}}}}^{{\rm {T}}}}{{\bf {R}}_{{\rm {d}}}}}\right)
\label{eq:w2}
\end{equation}
Then, 
\begin{equation}
{\dot{L}_{4}}=-\left\lVert {{\mathop{{\rm vex}}\nolimits}}\left({{\bf {R}}_{{\rm {d}}}^{{\rm {T}}}{\bf {R}}_{{\rm {b}}}^{{\rm {e}}}-{\bf {R}}{{_{{\rm {b}}}^{{\rm {e}}}}^{{\rm {T}}}}{{\bf {R}}_{{\rm {d}}}}}\right)\right\rVert ^{2} \le0.
\end{equation}

\subsection{Stability Analysis}

Based on the previous control objectives, the multicopter interceptor attitude loop controller is summarized as
{
\begin{equation}
\begin{small}
\left\{ \begin{aligned}^{{\rm {b}}}{{\bf {\omega}}_{{\rm {d}}}} & ={\rm {sat}}\left({}^{{\rm {b}}}{{\bf {\omega}}_{1}}+{}^{{\rm {b}}}{{\bf {\omega}}_{2}},\omega_{{\rm {m}}}\right)\\
{f_{{\rm {d}}}} & =\min\left(\max\left({{\bf {n}}_{{\rm {f}}}^{{\rm {T}}}\left({{\bf {a}}_{{\rm {d}}}} \!-\! m{{\bf {g}}} \!-\! {}^{{\rm {e}}}{\bf {f}}_{\rm{drag}}\right),0}\right),f_{{\rm {m}}}\right)
\end{aligned}
\right.
\end{small}
\label{eq:controller-summarized}
\end{equation}}
The multicopter's maximum angular velocity is represented by $\omega_{{\rm {m}}}$, where $\omega_{{\rm {m}}}>0$.
The function ${{\rm {sat}}\left(\cdot\right)}$ is a saturation function defined as \cite{Fu2023}
\begin{equation}
\begin{aligned}{\rm {sat}}\left({{\bf {a}},{a_{{\rm {m}}}}}\right) & ={\kappa_{{a_{{\rm {m}}}}}}\left({\bf {a}}\right){\bf {a}},\ {\bf {a}}\in{\mathbb{R}^{n}},\ {a_{{\rm {m}}}}\in{\mathbb{R}_{+}},\\
{\kappa_{{a_{{\rm {m}}}}}}\left({\bf {a}}\right) & =\left\{ {\begin{array}{cc}
{1,} & {\left\Vert {\bf {a}}\right\Vert \le{a_{{\rm {m}}}}}\\
{{{a_{{\rm {m}}}}} / {{\left\Vert {\bf {a}}\right\Vert }},} & {\left\Vert {\bf {a}}\right\Vert >{a_{{\rm {m}}}}}
\end{array}}\right..
\end{aligned}
\end{equation}

\begin{remark}
	The angular velocity control component ${}^{{\rm {b}}}{{\bf {\omega}}_{1}}$ is designed to keep the target within the FOV by aligning the designed LOS vector with the target unit vector. Another component, ${}^{{\rm {b}}}{{\bf {\omega}}_{2}}$, aligns the velocity vector with the target unit vector to ensure the interceptor's relative position converges to zero, achieving interception. \\
	\indent In practical applications, to ensure safety, saturation is applied to both angular velocity and thrust. The saturation function ${{\rm {sat}}\left(\cdot\right)}$ and the saturation of thrust alter only the magnitude of the control variables, not their direction, thereby affecting the convergence speed but not the system's stability.
	\vspace{-0.1cm}
\end{remark}

\begin{theorem}
	Under \textit{Assumption \ref{assumption:origin-coincides}}, for a multicopter with strapdown camera model \eqref{eq:rigid-body-model}--\eqref{eq:measurement model}, the attitude controller is designed by \eqref{eq:controller-summarized}. If the initial state satisfies ${\bf {n}}_{{\rm {t}}}\left(0\right)  \in \mathcal{S} = \left\{ {\bf {x}} \mid \left\lvert {{\bf{n}}_{{\rm{td}}}^{{\rm {T}}} {\bf {x}}} \right\rvert > 1 - k_{\rm{b}} \right\}$, then ${\bf {n}}_{{\rm {t}}} \in \mathcal{S}$, and ${\bf{p}}_{\rm{r}} \to 0$, as $t \to \infty$.
\end{theorem}

\begin{proof}
	See \textit{Appendix \ref{subsec:appendix-proof}}.
\end{proof}

\section{Delayed Kalman Filter Observer\label{sec:Delayed-Kalman-Filter}}

In this section, we design a state observer to address \textit{Problem \ref{problem:State-observer}}, using the multicopter's state and delayed image features to estimate current image features in real-time. Our method employs an extended Kalman filter with Gaussian distributions, where the state progresses according to kinematics equations without process noise, and the covariance follows its linearized form. The approach includes both propagation and correction steps. Two main challenges are integrating delayed image measurements with concurrent IMU data during correction, and handling the rotation matrix that evolves on the compact Lie group. These are tackled by treating the integration as an optimal prediction and representing the distribution using quaternions.

\subsection{Prediction\label{subsec:Prediction}}

The distribution at $t_{k-1}$ is given by ${\bf {\hat{x}}}_{k-1}$ and ${\bf {P}}_{k-1}$. 
The goal of the prediction is to propagate them to the next time step $t_{k}$ via \eqref{eq:rigid-body-model}--\eqref{eq:dotb_acc}.
The one-step prediction of state ${\bf {x}}$ is
\begin{equation}
{\bf {\hat{x}}}_{k|k-1}={\bf {F}}_{k}{\bf {\hat{x}}}_{k-1}.\label{eq:x_hat}
\end{equation}

As part of the propagation step, the covariance matrix ${\bf P}$ is updated by the well-known Kalman filter
\begin{equation}
{\bf {P}}_{k|k-1}={\bf {F}}_{k}{\bf {P}}_{k-1}{{\bf {F}}_{k}^{{\rm {T}}}}+{\bf {G}}_{k}{\bf {Q}}_{k}{{\bf {G}}_{k}^{{\rm {T}}}}\label{eq:P_k_k-1}
\end{equation}
where ${\bf {F}}_{k}\in\mathbb{{R}}{}^{18\times18}$ and ${\bf {G}}_{k}\in\mathbb{{R}}{}^{18\times6}$
are the state transition matrix and the noise matrix from time ${{t_{k-1}}}$
to time ${{t_{k}}}$.

According to the continuous form of the state vector relation (\ref{eq:rigid-body-model}, \ref{eq:IBVS}--\ref{eq:p^bar_img,k}), the partial derivative of the state vector ${{{\bf {x}}_{k}}}$ at time $t_{k}$ with respect to the state vector ${{{\bf {x}}_{k-1}}}$ at time $t_{k-1}$ is obtained as ${\bf {F}}_{k}$ and ${\bf {G}}_{k}$,

\[
\setlength{\arraycolsep}{2.5pt}\begin{aligned}{\bf {F}}_{k}\! & =\!\left[{\begin{array}{cccccc}
{{\bf {F}}_{{\bf {q}}_{k-1}}^{{\bf {q}}_{k}}} & {\bf {0}} & {\bf {0}} & {\bf {0}} & {{\bf {F}}_{{{\bf {b}}_{{\rm {gyr}},k-1}}}^{{\bf {q}}_{k}}} & {\bf {0}}\\
{\bf {0}} & {{\bf {I}}_{3\times3}} & {{\bf {F}}_{{{\bf {v}}}_{{\rm{r}},{k-1}}}^{{{\bf {p}}}_{{\rm{r}},{k}}}} & {\bf {0}} & {\bf {0}} & {\bf {0}}\\
{{\bf {F}}_{{\bf {q}}_{k-1}}^{{{\bf {v}}}_{{\rm{r}},{k}}}} & {\bf {0}} & {{\bf {I}}_{3\times3}} & {\bf {0}} & {\bf {0}} & {{\bf {F}}_{{{\bf {b}}_{{\rm {acc}},k-1}}}^{{{\bf {v}}}_{{\rm{r}},{k}}}}\\
{{\bf {F}}_{{\bf {q}}_{k-1}}^{^{{\rm {i}}}{\bf {\bar{p}}}_{k}}} & {\bf {0}} & {{\bf {F}}_{{{\bf {v}}}_{{\rm{r}},{k-1}}}^{^{{\rm {i}}}{\bf {\bar{p}}}_{k}}} & {{\bf {F}}_{^{{\rm {i}}}{\bf {\bar{p}}}_{k-1}}^{^{{\rm {i}}}{\bf {\bar{p}}}_{k}}} & {{\bf {F}}_{{{\bf {b}}_{{\rm {gyr}},k-1}}}^{^{{\rm {i}}}{\bf {\bar{p}}}_{k}}} & {\bf {0}}\\
{\bf {0}} & {\bf {0}} & {\bf {0}} & {\bf {0}} & {{\bf {I}}_{3\times3}} & {\bf {0}}\\
{\bf {0}} & {\bf {0}} & {\bf {0}} & {\bf {0}} & {\bf {0}} & {{\bf {I}}_{3\times3}}
\end{array}}\right],\\
{\bf {G}}_{k} & =\left[{\begin{array}{cc}
{{\bf {G}}_{^{{\rm {b}}}{\bf {\omega}}_{k-1}}^{{\bf {q}}_{k}}} & {\bf {0}}\\
{\bf {0}} & {\bf {0}}\\
{\bf {0}} & {{\bf {G}}_{^{{\rm {e}}}{\bf {a}}_{k-1}}^{{{\bf {v}}}_{{\rm{r}},{k}}}}\\
{{\bf {G}}_{^{{\rm {b}}}{\bf {\omega}}_{k-1}}^{^{{\rm {i}}}{\bf {\bar{p}}}_{k}}} & {\bf {0}}\\
{\bf {0}} & {\bf {0}}\\
{\bf {0}} & {\bf {0}}
\end{array}}\right].
\end{aligned}
\]

The relationships between orientation quaternion, position, velocity, image features, gyroscope bias, accelerometer bias, and state are derived in \textit{Appendix \ref{subsec:appendix-DKF-Propagation}}.
The state transfer and measurement equations introduced by image measurement are discussed below.

\subsection{Correction with Delayed Image Measurement}

\begin{figure}
	\centering{}\includegraphics{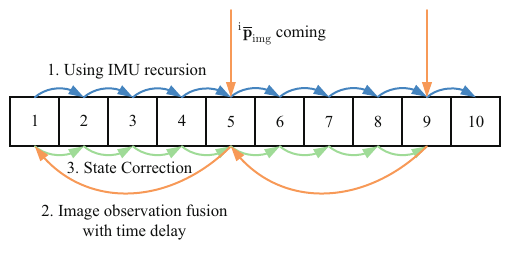}
	
	\caption{\label{fig:Data-transfer}Data flow diagram. Each block represents
	the data saved between the delayed image measurements.}
	
\end{figure}

Image processing provides the image coordinates of the target, and the state ${^{{\rm {i}}}{\bf {\bar{p}}}}$ is obtained after normalization. Since camera imaging and image processing take time, the measurement is obtained from a period ago. 
Suppose the delay time is ${{t_{D}}>0}$, then the image measurement at time $t$ is expressed as Eq. \eqref{eq:p^bar_img,k}. After discretization, the measurement equation is
\begin{equation}
{\bf {z}}_{k-D}={\bf {H}}{\bf {x}}_{k-D}+{\bf {v}}
\end{equation}
where ${{\bf {z}}={\left[{\begin{array}{cc}
{{\bar{p}}_{{x_{{\rm {i}}}}}} & {{\bar{p}}_{{y_{{\rm {i}}}}}}\end{array}}\right]^{{\rm {T}}}}\in{\mathbb{R}^{2}}}$, ${{\bf {H}}}$ is the measurement matrix; ${\bf {v}}$
is the system noise vector, in which the elements follow a Gaussian
distribution with zero mean values that are not correlated with each
other, ${\bf {v}}\sim{\cal N}\left({{{\bf {0}}_{2\times1}},{\bf {R}}}\right),{\bf {R}}\in\mathbb{{R}}{}^{2\times2}$.
Due to the measurement delay, ${\bf {z}}_{k-D}$ is obtained at time
$t_{k}$. The measurement matrix ${{\bf {H}}}$ is
\begin{equation}
{\bf {H}}=\left[{\begin{array}{cccccc}
{{\bf {0}}_{2\times4}} & {{\bf {0}}_{2\times3}} & {{\bf {0}}_{2\times3}} & {{\bf {I}}_{2\times2}} & {{\bf {0}}_{2\times3}} & {{\bf {0}}_{2\times3}}\end{array}}\right].\label{eq:Measurement equation}
\end{equation}

This essentially estimates the current state by utilizing a past prior state estimate and a past measurement. More explicitly, it is solved by (i) correcting the prior state at time $t_{k-D}$ using the current measurement ${{\bf {z}}_{k-D}}$, and (ii) propagating it to time $t_{k}$ by repeating the IMU prediction step. In step (ii), IMU prediction and propagation are presented in detail as follows. 

The delay processing is shown in Fig. \ref{fig:Data-transfer}, and the image measurement delay is spread over four IMU measurement cycles. The main point of discussion here is how the currently measured image features are used to correct the historical state estimate and to estimate the current state further.

From the measurement Eq. \eqref{eq:Measurement equation}, the
estimated state is
\begin{equation}
{\bf {\hat{x}}}_{k-D}\!=\!{\bf {\hat{x}}}_{{k-D|k-D-1}}\!+\!{\bf {K}}\left({{\bf {z}}_{k-D}\!-\!{\bf {H}}{\bf {\hat{x}}}_{{k-D|k-D-1}}}\right).\label{eq:x_hat_k-D}
\end{equation}
The covariance ${\bf {P}}_{k-D}$ of the state ${\bf {x}}_{k-D}$
is given by
\begin{equation}
{\bf {P}}_{k-D}=\left({{\bf {I}}-{\bf {K}}{\bf {H}}}\right){\bf {P}}_{k-D|k-D-1}\label{eq:P_k-D}
\end{equation}
where ${\bf {K}}$ is the filter gain, 
\begin{equation}
\begin{aligned}{\bf {K}} & ={\bf {P}}_{k-D|k-D-1}{{\bf {H}}^{{\rm {T}}}}{{\bf {S}}^{-1}},\\
{\bf {S}} & ={\bf {H}}{\bf {P}}_{k-D|k-D-1}{{\bf {H}}^{{\rm {T}}}}+{\bf {R}}.
\end{aligned}
\label{eq:KandS}
\end{equation}
After the above steps, the optimal state estimation at time $t_{k-D}$
is obtained, and then the state is propagated to the current time
$t_{k}$ according to the Eqs. \eqref{eq:x_hat}--\eqref{eq:P_k_k-1}.

\section{Hardware-in-the-Loop Simulations\label{sec:HITL-simulation}}

In this section, the proposed interception controller and DKF observer are implemented in the Hardware-in-the-Loop (HITL) simulations to verify the algorithm's effectiveness and perform statistical analysis with a large number of simulations.

\subsection{Introduction to Simulation Platform}

\begin{figure*}
	\centering{}\includegraphics[width=0.85\textwidth,height=0.65\textheight]{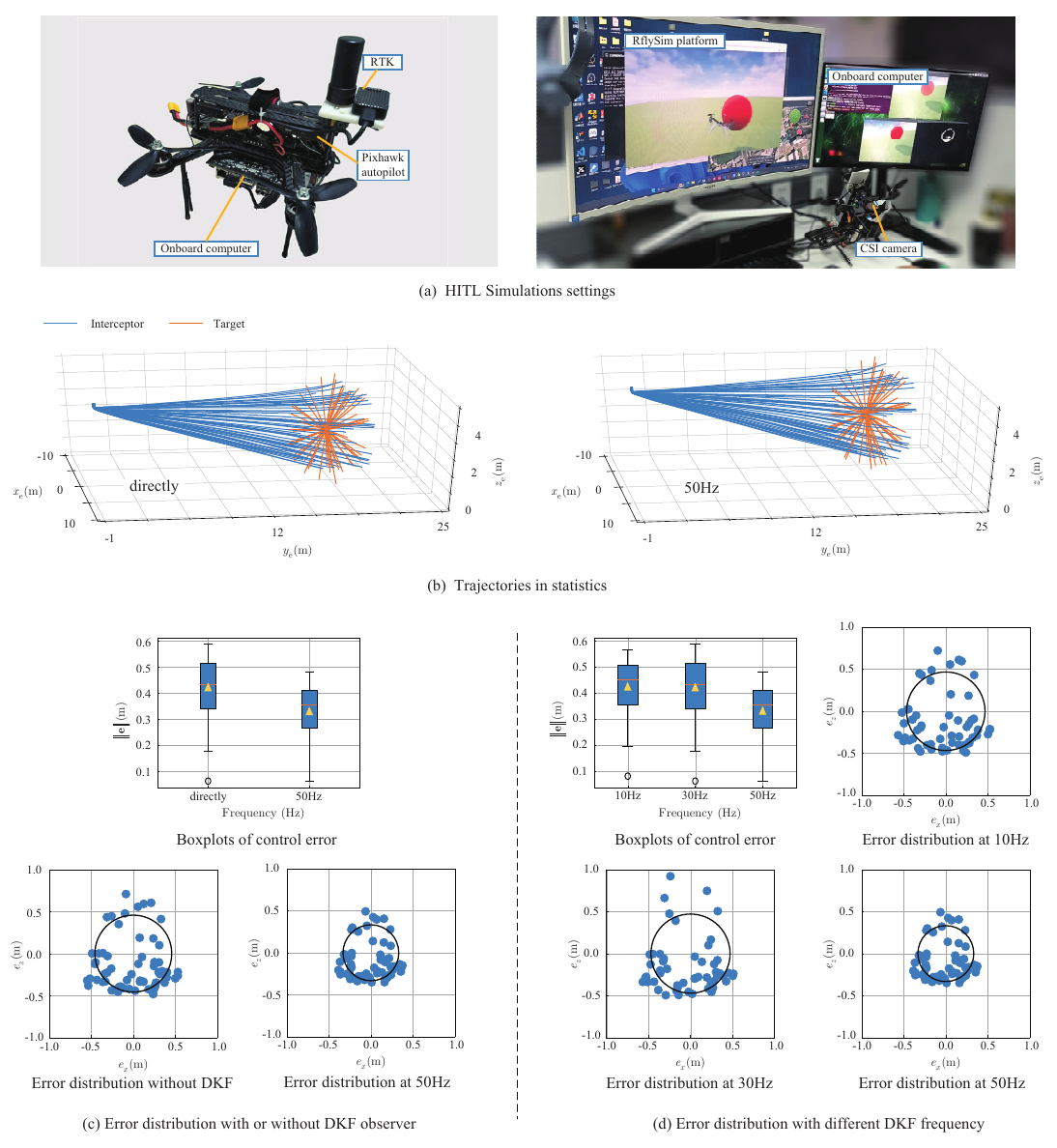}\caption{\label{fig:HITLInterception-statistic}{HITL simulation statistics.} 
	}
\vspace{-0.4cm}
\end{figure*}

\begin{figure*}
	\centering{}\includegraphics[width=0.85\textwidth,height=0.4\textheight]{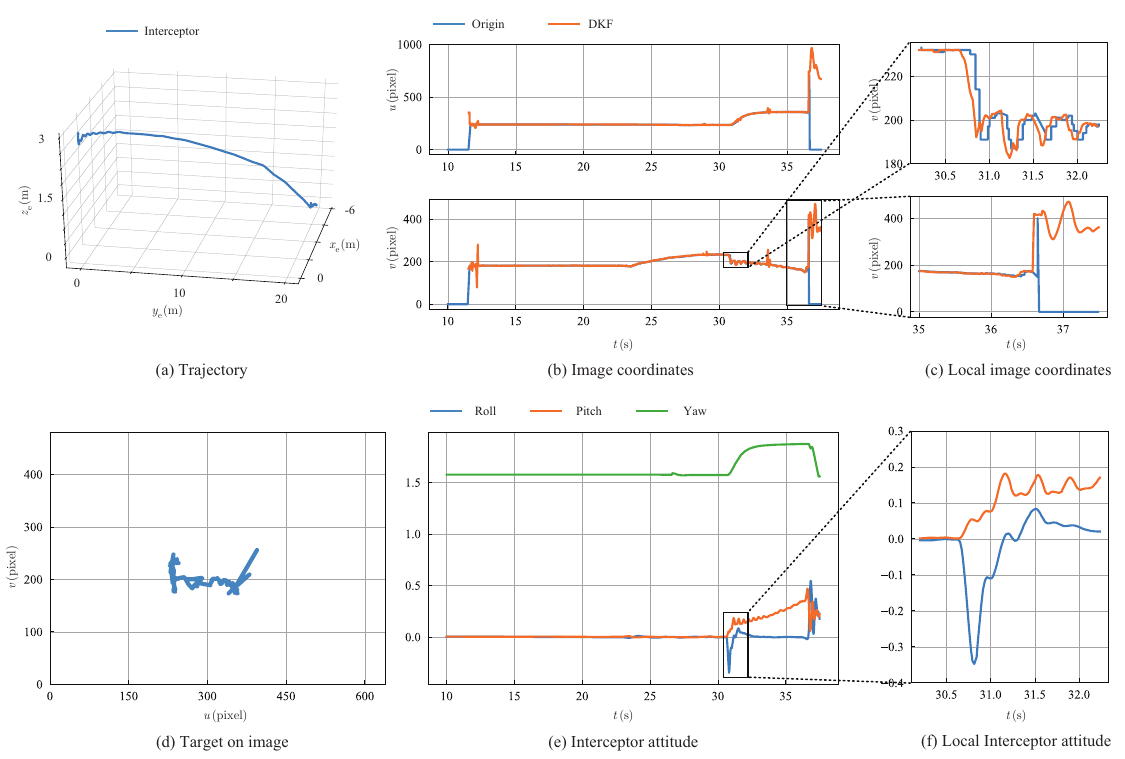}\caption{\label{fig:HITLInterception-detail}{HITL simulation process.} 
	}
\vspace{-0.4cm}
\end{figure*}

RflySim is a control and safety testing ecosystem for unmanned systems, utilizing Model-Based Design (MBD) to verify algorithms and system applicability through high-fidelity digital scenes and a UAV model library \cite{dai2021rflysim}.
Utilizing automated test software, RflySim facilitates extensive simulation verification based on metrics like Circular Error Probable (CEP), which is defined as the radius of a circle centered on a target, within which 50\% of projectiles are expected to land \cite{Shnidman1995CEP}.

RflySim is also adept at supporting HITL simulations. 
As shown in Fig. \ref{fig:HITLInterception-statistic}(a), HITL simulations use the same hardware and software architecture as the real flight experiments, including the onboard computer, the Pixhawk autopilot, and the high-speed interception algorithm.
The key difference between HITL simulations and real flights lies in the source of sensor data: in HITL, images and IMU data is sourced from the RflySim platform instead of the real world.

During HITL simulations, the onboard computer processes simulated image and IMU data from RflySim, executing state estimation and control signal calculations in real-time. These control signals are sent to the Pixhawk, which interacts with RflySim to simulate the interceptor's motion. RflySim then generates and sends new image and IMU data to the onboard computer, completing a closed-loop data and control system.

\subsection{Image Filtering Effect Experiments}

\subsubsection{Filtering Algorithm}

\begin{algorithm}[t]
	\caption{\label{alg:Filtering}${{\bf {\hat{x}}}}$ = DelayedFilter(${{{\bf {\omega}}_{{\rm {gyr}}}}}$,
	${{{\bf {a}}_{{\rm {acc}}}}}$, ${^{{\rm {i}}}{{\bf {\bar{p}}}_{{\rm {img}}}}}$)}
	
	\begin{algorithmic}[1] 
	\Require{IMU measurement, image measurement} 
	\Ensure{Estimated state} 
	\State{Initialize ${{\bf{\hat x}}_0}$, ${{\bf{P}}_0}$, ${{\bf{K}}_0}$}
	\For{$t=1:n$}
		\If{${}^{\rm{i}}{{\bf{\bar p}}_{{\rm{img}}}}$ is not available}
			\State{Follow \eqref{eq:x_hat} and \eqref{eq:P_k_k-1} to recursive state ${{\bf{\hat x}}_t}$ with IMU data}
		\Else
			\State{According to \eqref{eq:x_hat_k-D}, \eqref{eq:P_k-D}, \eqref{eq:KandS}, estimate the state ${{\bf{\hat x}}_{t - D}}$ at time ${t_{k - D}}$}
			\For{$i=D-1:0$}
				\State{Follow \eqref{eq:x_hat} and \eqref{eq:P_k_k-1} to recursive state ${{\bf{\hat x}}_t}$ with stored IMU data}
			\EndFor
		\EndIf
	\EndFor
	\end{algorithmic}
\end{algorithm}

\textit{Algorithm \ref{alg:Filtering}} outlines our state estimation method, and Fig. \ref{fig:Data-transfer} illustrates data transmission. High-frequency, real-time IMU data is used for state recursion when image measurements are pending. Once a delayed image measurement arrives, it corrects the state variable from $D$ cycles prior, updating it to the current moment for enhanced accuracy. C++ and Eigen\footnote{https://eigen.tuxfamily.org/} are employed for real-time computation efficiency.

\subsubsection{Analysis of Influence of Delay on Control Accuracy}

This section evaluates the DKF observer's impact on interception control accuracy. Simulations on RflySim, run 50 times each scenario  with the IMU and controller at \SI{200}{\hertz} and target perception at \SI{30}{\hertz}, compare interception accuracies without and with the DKF observer.
Fig. \ref{fig:HITLInterception-statistic}(b) shows 50 trial trajectories without and with the DKF observer, showing smoother interceptor movements and earlier aggressive maneuvers due to image prediction, enhancing interception success rates.

In the simulations, two experiment sets are conducted: (i) direct control using target detection and (ii) control using DKF-based target estimation.
The same random seed is set to perform 50 experiments in both experiment sets, and the control error box plot is shown in Fig. \ref{fig:HITLInterception-statistic}(c). 
The results and circular probability errors are shown in error distribution plots, specifically, ${\rm{CEP}}_{\rm{direct}}=0.408 \,\rm{m}$, ${\rm{CEP}}_{\rm{50Hz}}=0.332\,\rm{m}$.
It is shown that the DKF observer can effectively improve the control effect.

\subsubsection{Analysis of Influence of DKF Frequency}
This section investigates the effect of the IMU frequency on the accuracy of the interception control. In all simulations, the IMU frequency and inner-loop controller are set to \SI{200}{\hertz}. 
The DKF observer run frequency is set to \SI{10}{\hertz}, \SI{30}{\hertz}, and \SI{50}{\hertz}, respectively, and each set of simulations is run 50 times to obtain the CEP.

The same random seed is set in the three experiment sets, and the control error box plot is shown in Fig. \ref{fig:HITLInterception-statistic}(d). 
The results and circular probability errors are shown in error distribution plots, specifically, ${\rm{CEP}}_{\rm{10Hz}}=0.425\,\rm{m}$, ${\rm{CEP}}_{\rm{30Hz}}=0.423\,\rm{m}$, ${\rm{CEP}}_{\rm{50Hz}}=0.332$\,\rm{m}.
It is shown that the increased frequency of the DKF observer is beneficial to improve the control effect. Therefore, for practicality and computational efficiency, the DKF observer frequency is set at \SI{50}{\hertz} in real flight experiments.

\subsubsection{Details of DKF Observer}
In a simulation with the DKF observer set at \SI{50}{\hertz}, one of 50 trials is depicted in Fig. \ref{fig:HITLInterception-detail}. The interceptor multicopter initially ascends to \SI{3}{m} before starting its mission, adjusting its motion direction early in the interception and constantly refining its speed, direction, and attitude. The image coordinates plot and the local image coordinates plot illustrate the target's pixel coordinates and the filter-estimated coordinates from target acquisition to interception completion. Despite an \SI{80}{ms} image delay, the DKF observer effectively estimates the actual image coordinates, quickly converging early in the filtering process and aligning well with the measured values' overall trend.

It is shown that: (i) the filter estimated value changes before the observed value in the rising and falling segments, showing an obvious delay compensation effect; (ii) the filter's update frequency is higher than the image measurement update frequency; (iii) after the intercept was completed at \SI{36.6}{s}, the filter continued to estimate based on the last known target position.
These aspects enhance control quality, with the IMU compensating for image delays and the higher update frequency allowing for more current target estimations for control calculations. Additionally, if the target is lost mid-interception, the interceptor has a chance to reacquire it. Overall, the DKF observer successfully mitigates the impact of image delay on control, increasing confidence for its use in real flight experiments.

\subsection{Experiments comparing with previous IBVS method}
{To validate our proposed method, we compared it with the previously established IBVS method \cite{yang2020autonomous} by conducting 50 simulation experiments under the same initial conditions. The main difference in this study was the controller, while the DKF observer was consistently set at \SI{50}{\hertz} in both cases.}

\begin{figure*}
	\centering{}\includegraphics[width=0.85\textwidth,height=0.35\textheight]{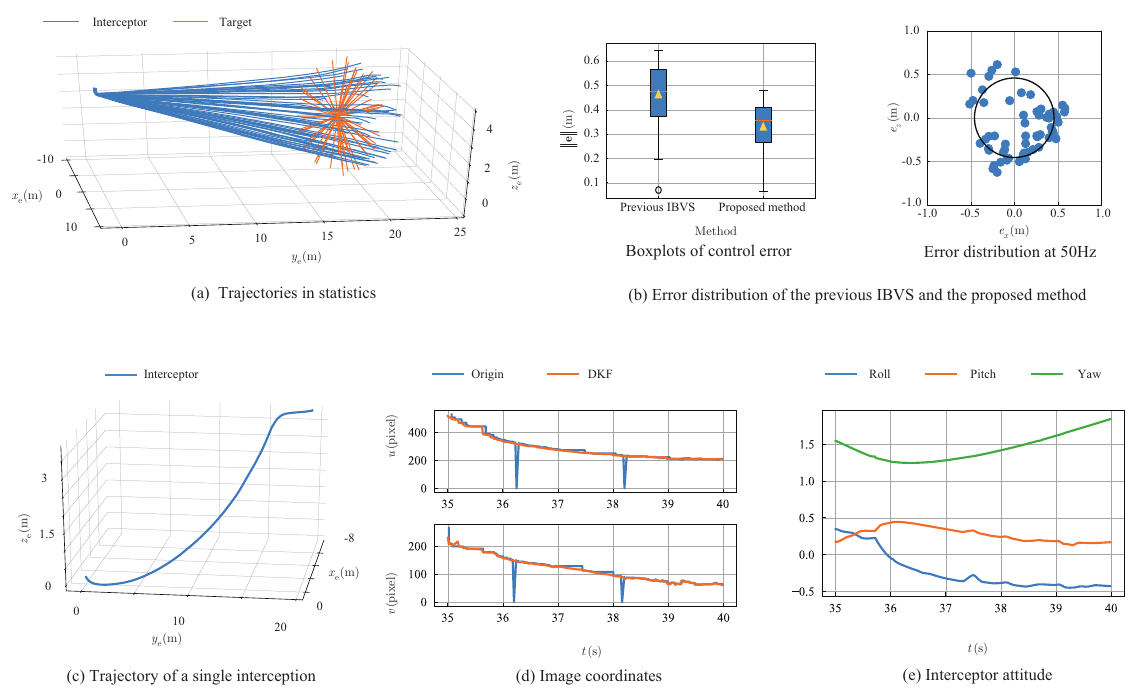}\caption{\label{fig:Compare-previous-IBVS}{Statistical results and single intercept details for comparison with previous IBVS methods.}
	}
\end{figure*}

{Fig. \ref{fig:Compare-previous-IBVS} presents the results of these comparative experiments. Fig. \ref{fig:Compare-previous-IBVS}(a) illustrates the aggregated trajectory data from all 50 experiments. A direct comparison between Fig. \ref{fig:Compare-previous-IBVS}(b) and Fig. \ref{fig:HITLInterception-statistic}(d) reveals a significant improvement in the interception accuracy and a denser clustering of interception points in our proposed method, supported by the CEP values of ${\rm{CEP}}_{\rm{previous}}=\SI{0.457}{m}$ versus ${\rm{CEP}}_{\rm{50Hz}}=\SI{0.332}{m}$.}

{Additional plots in Figs. \ref{fig:Compare-previous-IBVS}(c), (d), and (e) detail one experimental trajectory, image coordinates over time, and interceptor attitude, respectively. These plots underscore a key advantage of our method: the ability to induce more agile attitude adjustments. Such agility is crucial for enhancing the precision of interception, as demonstrated by the reduced CEP value. Overall, this comparative analysis underscores the enhanced interception accuracy and responsiveness of our approach compared to the previous IBVS method.}

\begin{figure*}
	\centering{}\includegraphics[width=0.85\textwidth,height=0.55\textheight]{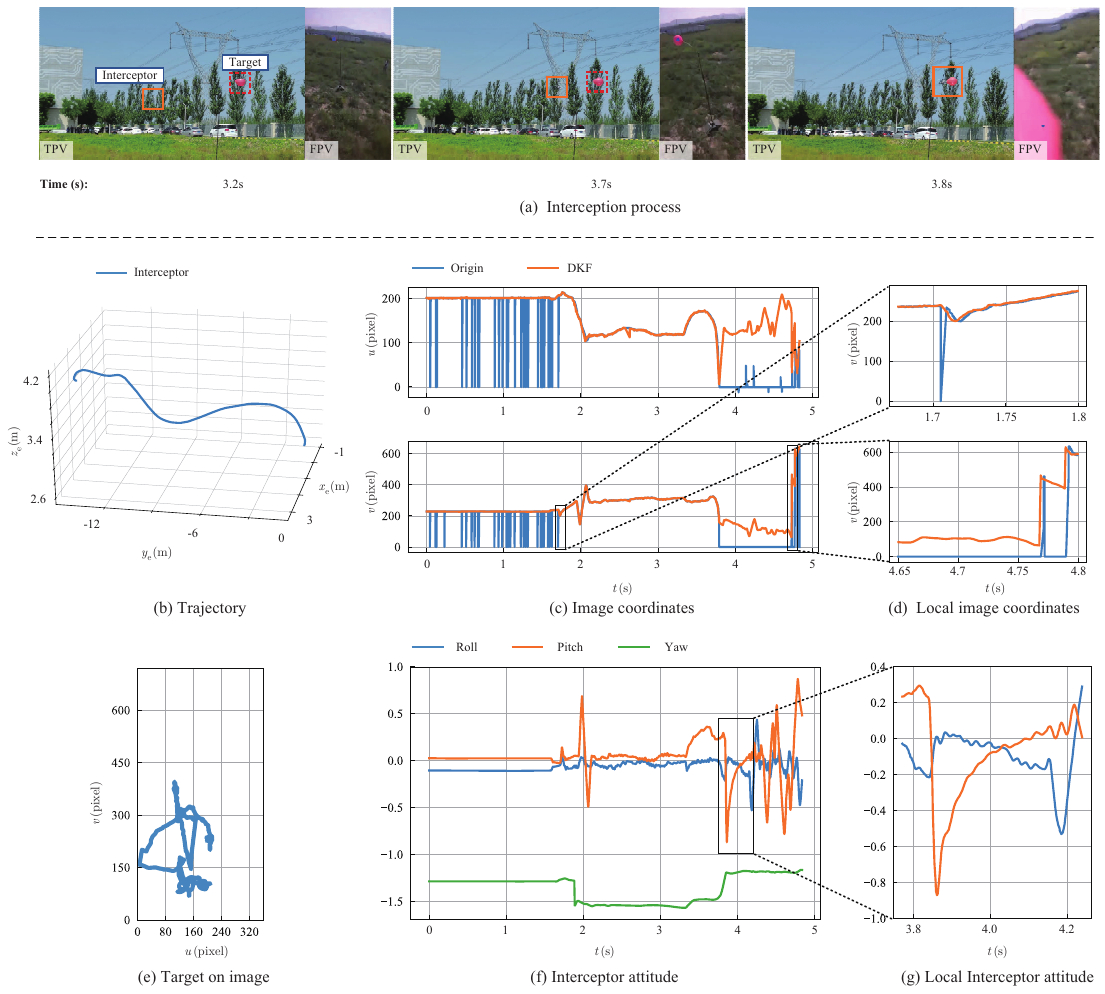}\caption{\label{fig:Interception-process}{Interception process of static target. 
	In figure (a), the left is the Third-Person View (TPV), and the right is the First-Person View (FPV) of the onboard camera.} 
	}
\end{figure*}

\section{High-speed Interception Flight Experiments\label{sec:flight-experiment}}

Extensive HITL simulations verify the effectiveness and robustness of our proposed method. Building on these simulations, real flight experiments were conducted, mirroring the steps of the simulation process to actualize the experiments.

\subsection{Static Target Flight Experiments}

\subsubsection{Design of Flight Experiments}

The real flight experiments used the same multicopter as the HITL simulations shown in Fig. \ref{fig:HITLInterception-statistic}(a), equipped with a T-Motor propulsion system, Pixhawk Nano V5 flight controller, NVIDIA Jetson Xavier NX onboard computer, and a front CSI camera with a \ang{120} FOV. The multicopter has a thrust-to-weight ratio of about 3, can accelerate up to \SI{5}{m/s^2}, and has a 20-minute flight time. The target, a red balloon, was placed \SI{3}{m} above ground level in open space, as illustrated in Fig. \ref{fig:Interception-process}.

The multicopter intercepted the balloon at varying speeds, \SIlist{5;10;15;19;21}{m/s}, as shown in Fig. \ref{fig:Interception-process}(a).
It can be observed that the multicopter kept approaching the target and kept the target in the FOV until the final interception task was completed. During the experiment, the interceptor's inclination reached \ang{50}, and the maximum speed was \SI{21}{m/s}. Experiment videos are available at \url{https://youtu.be/146Km6c30Ww}.

\subsubsection{Flight Data Analysis}

{Fig. \ref{fig:Interception-process}(b) shows the interceptor's trajectory, with GPS data used for position and velocity analysis.}
Note that the designed interception controller does not need to use GPS position data. 
{In the flight trajectory plot, the start position of the multicopter is about $(\SI{-1}{m},\SI{-1.8}{m},\SI{3}{m})$, and the target position is about $(\SI{-3}{m},\SI{-12}{m},\SI{4}{m})$.} The multicopter smoothly adjusted its attitude in the initial interception stages and continuously adapted its attitude and acceleration for accurate final interception.

Fig. \ref{fig:Interception-process}(c) displays the time-varying curve of estimated and measured target coordinates during flight. The DKF filter effectively estimates the target position even when the target is momentarily lost or leaves the FOV, as seen in the local image coordinates plot in Fig. \ref{fig:Interception-process}(d). The filter compensates for image measurement delay and maintains target estimation even with extended target loss.

{The movement of the target on the image is depicted in Fig. \ref{fig:Interception-process}(e), where the controller effectively maintains the target within the image's central area. The interceptor's attitude variation over time is illustrated in Fig. \ref{fig:Interception-process}(f)--(g), with the maximum pitch angle reaching \ang{50} during the interception process.}

\subsection{Moving Target Flight Experiments}

\begin{figure*}
	\centering{}\includegraphics[width=0.85\textwidth,height=0.55\textheight]{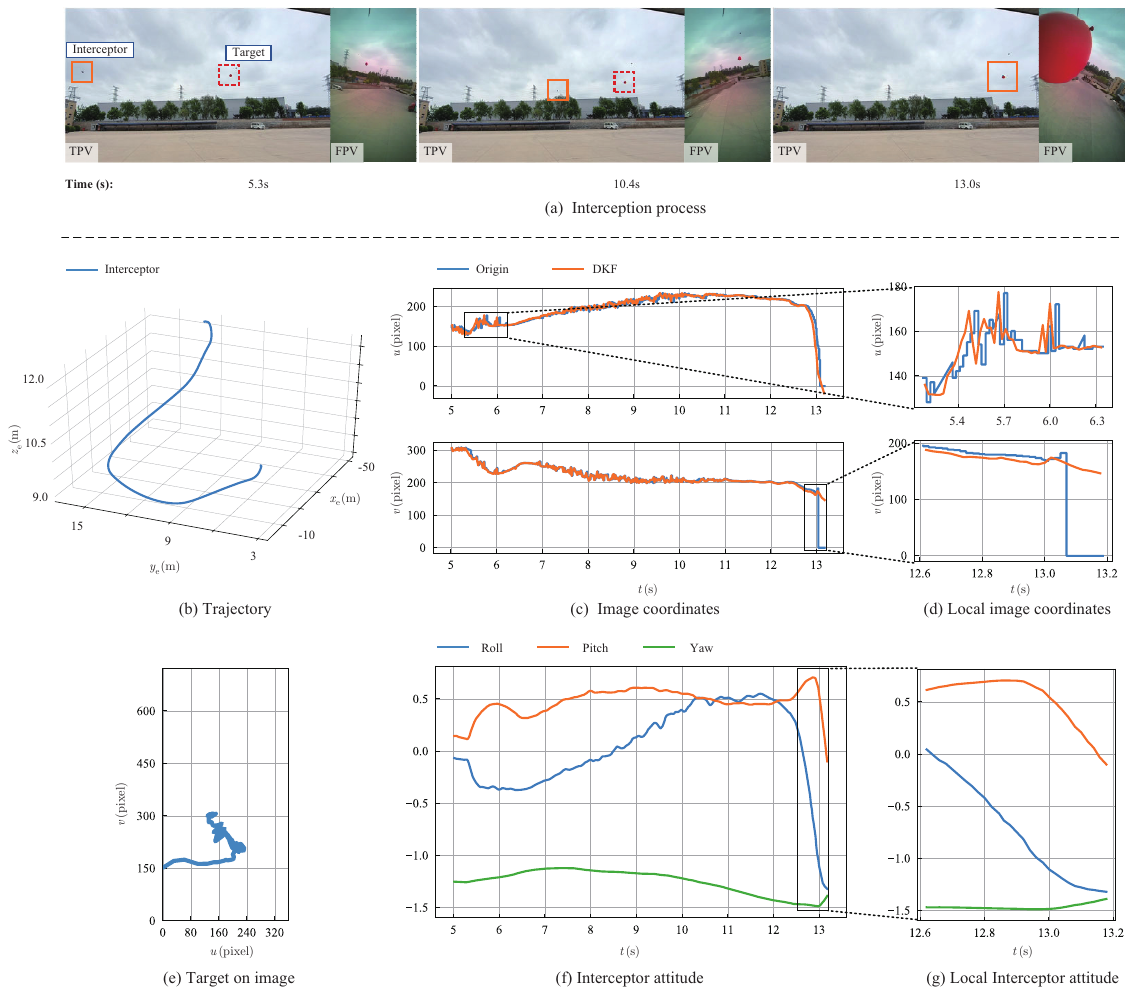}\caption{\label{fig:Interception-process-moving}{Interception process of moving target. }
	}
	\vspace{-0.4cm}
\end{figure*}

\subsubsection{Design of Flight Experiments}

In the experiments illustrated in Fig. \ref{fig:Interception-process-moving}, a target attached to a multicopter moves in a circular path at \SI{6}{m/s}, with its motion affected by wind. Fig. \ref{fig:Interception-process-moving}(a) shows the interceptor tracking and approaching the irregularly moving target, keeping it within its FOV until interception. The interceptor's trajectory, detailed in Fig. \ref{fig:Interception-process-moving}(b), displays effective target locking and agile movements, reaching speeds up to \SI{10.5}{m/s}. Additional experiments involving a freely floating helium balloon are included in the linked video.

\subsubsection{Flight Data Analysis}

The recorded data of a flight experiment using the filtered estimates are shown in Fig. \ref{fig:Interception-process-moving}(c)--(e).
The image coordinates plot reveals the DKF filter's effectiveness in smoothing image measurement anomalies and correcting delay. Despite the delay, the estimation remained accurate without impacting autonomous flight performance. The interceptor's attitude variations, depicted in Fig. \ref{fig:Interception-process-moving}(f)--(g), demonstrate its excellent maneuverability.

In summary, the DKF-based state observer compensates for delays in perception module imaging and processing. This integration allows the controller to access more accurate real-time image measurements, enhancing control performance despite lower imaging frame rates and potential target loss.

\subsection{Compared Flight Experiments}

\begin{figure}
	\centering{}\includegraphics{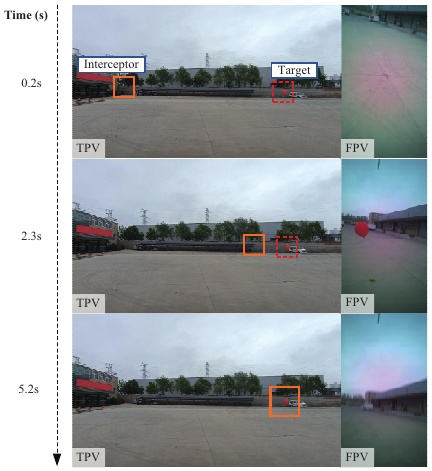}\caption{\label{fig:Compared}Interception process of PBVS. 
	}
\end{figure}

For comparison, we implemented simple experiments with a Position-based Visual Servo (PBVS) controller following the pipeline of \cite{stasinchuk2021multi}. As shown in Fig. \ref{fig:Compared}, the helium-filled balloon floats about \SI{2}{m} above the ground.
It can be observed that the interceptor misses the target due to two reasons: (i) the PBVS controller's inability to maintain target lock in the image, and (ii) wind disturbances leading to inaccurate and delayed 3-D target position estimation.

\section{Conclusion\label{sec:conclusion}}

This paper presents a high-speed autonomous interception method for UAV intrusion, effectively overcoming the limitations of existing radio frequency interference techniques. The main results include:
(i) Design of an IBVS interception controller, ensuring collinearity between the designed LOS vector, the velocity vector, and the target unit vector, with proven stability of the angular velocity controller.
(ii) Development of a DKF observer to address issues of delayed, low frame rate, and easily loss image feedback, enhancing feedback data quality.
(iii) Successful validation of the method through unique, high-speed autonomous flight experiments and HITL simulations, achieving speeds over \SI{20}{m/s}.
Future work will focus on refining the interception strategy using game theory, considering the strategies of intruding drones, and exploring collaborative interceptions involving multiple interceptors and multiple intruders.

\appendix

\subsection{Detailed Derivations and Proof of Theorem 1\label{subsec:appendix-proof}}
{
\subsubsection{Derivatives of Key Elements}
First, the derivatives of its elements are considered to analyze the properties of the Lyapunov candidate function.
By the property of the derivative of ${\mathsf{SO}}\left(3\right)$ \cite{murray2017mathematical}, ${{{\dot {\bf{n}}}_{{\rm {td}}}}={}^{{\rm {e}}}{\bf {\omega}}\times{\bf{n}}_{{\rm{td}}}}$, where $^{{\rm {e}}}{\bf {\omega}}$ is the angular velocity under EFCS. The  vector 2-norm is defined as $\left\Vert {\bf {p}}_{\rm{r}}\right\Vert = \sqrt{{\bf {p}}_{\rm{r}}^{{\rm {T}}}{\bf {p}}_{\rm{r}}} $, so the derivatives of ${{\bf {n}}_{{\rm {t}}}}$ and $z_1$ are
\begin{equation}
\begin{aligned}{{\dot {\bf{n}}}_{{\rm {t}}}} & =-\frac{{{\dot {\bf{p}}}_{\rm{r}}\left\Vert {\bf {p}}_{\rm{r}}\right\Vert -{\bf {p}}_{\rm{r}}{{\left({{{\bf {p}}_{\rm{r}}^{{\rm {T}}}}{\bf {p}}_{\rm{r}}}\right)}^{-\frac{1}{2}}}{{\bf {p}}_{\rm{r}}^{{\rm {T}}}}{\dot {\bf{p}}}_{\rm{r}}}}{{{\left\Vert {\bf {p}}_{\rm{r}}\right\Vert }^{2}}}\\
 & ={-\frac{{\dot {\bf{p}}}_{\rm{r}}}{{\left\Vert {\bf {p}}_{\rm{r}}\right\Vert }}+\frac{{{\bf {p}}_{\rm{r}}{{\bf {p}}_{\rm{r}}^{{\rm {T}}} {\dot {\bf{p}}}_{\rm{r}} }}}{{{\left\Vert {\bf {p}}_{\rm{r}}\right\Vert }^{3}}}}\\
 & =\left({-\frac{1}{{\left\Vert {\bf {p}}_{\rm{r}}\right\Vert }}{\bf {I}}+\frac{{{\bf {p}}_{\rm{r}}{{\bf {p}}_{\rm{r}}^{{\rm {T}}}}}}{{{\left\Vert {\bf {p}}_{\rm{r}}\right\Vert }^{3}}}}\right){\dot {\bf{p}}}_{\rm{r}}\\
 & =\left({-{\bf {I}}+{{\bf {n}}_{{\rm {t}}}}{\bf {n}}_{{\rm {t}}}^{{\rm {T}}}}\right)\frac{1}{{\left\Vert {\bf {p}}_{\rm{r}}\right\Vert }}{{\bf {v}}_{\rm{r}}},
\end{aligned}
\label{eq:no_dot}
\end{equation}
\begin{equation}
	\begin{aligned}{{\dot{z}}_{1}} & =-\left({{\bf {n}}_{{\rm {td}}}^{{\rm {T}}}{{\dot {\bf{n}}}_{{\rm {t}}}}+{\bf {n}}_{{\rm {t}}}^{{\rm {T}}}{{\dot {\bf{n}}}_{{\rm {td}}}}}\right)\\
	 & =-\frac{1}{{\left\Vert {\bf {p}}_{\rm{r}}\right\Vert }}{\bf {n}}_{{\rm {td}}}^{{\rm {T}}}\left({-{\bf {I}}+{{\bf {n}}_{{\rm {t}}}}{\bf {n}}_{{\rm {t}}}^{{\rm {T}}}}\right){{\bf {v}}_{\rm{r}}}-{\bf {n}}_{{\rm {t}}}^{{\rm {T}}}\left({^{{\rm {e}}}{\bf {\omega}}\times{\bf{n}}_{{\rm{td}}}}\right)\\
	 & =-\frac{1}{{\left\Vert {\bf {p}}_{\rm{r}}\right\Vert }}{\bf {n}}_{{\rm {td}}}^{{\rm {T}}}\left({-{\bf {I}}+{{\bf {n}}_{{\rm {t}}}}{\bf {n}}_{{\rm {t}}}^{{\rm {T}}}}\right){{\bf {v}}_{\rm{r}}}-{\left( {{{\bf{n}}_{{\rm{td}}}} \!\times\! {{\bf{n}}_{\rm{t}}}} \right)^{\rm{T}}}{}^{\rm{e}}\omega.
	\end{aligned}
	\label{eq:z1_dot_appendix}
\end{equation}}

{\subsubsection{Derivative of $L_1$}
The derivative of Eq. \eqref{eq:L1} along Eq. \eqref{eq:z1_dot_appendix} is
\begin{equation}
	\begin{small}
		\begin{aligned}{{\dot{L}}_{1}} & =\frac{z_1 \dot{z}_1}{{k_{\rm{b}}^2 - z_1^2}}\\
			& =- \frac{{{z_1}}}{{k_{\rm{b}}^2 \!-\! z_1^2}}\frac{1}{{\left\| {{{\bf{p}}_{\rm{r}}}} \right\|}}{\bf{n}}_{{\rm{td}}}^{\rm{T}}\left( { - {\bf{I}} \!+\! {{\bf{n}}_{\rm{t}}}{\bf{n}}_{\rm{t}}^{\rm{T}}} \right){{\bf{v}}_{\rm{r}}} \!-\! \frac{{{z_1}}}{{k_{\rm{b}}^2 - z_1^2}}{\left( {{{\bf{n}}_{{\rm{td}}}} \!\times\! {{\bf{n}}_{\rm{t}}}} \right)^{\rm{T}}}{}^{\rm{e}}\omega.
		   \end{aligned}
	\end{small}
	\label{eq:L1_dot_appendix}
\end{equation}
Due to ${\bf {v}}_{\rm{r}} = {\bf {v}}_{\rm{rd}} + {\bf {z}}_2 = -k_1 {\bf {p}}_{\rm{r}} + {\bf {z}}_2 = k_1 \left\lVert {\bf {p}}_{\rm{r}}\right\rVert {{\bf {n}}_{{\rm {t}}}} + {\bf {z}}_2$, Eq. \eqref{eq:L1_dot_appendix} is rewritten as
\begin{equation}
	\begin{small}
		\begin{aligned}{{\dot{L}}_{1}} 
		 & =\frac{{{z_1}}}{{k_{\rm{b}}^2 \!-\! z_1^2}} \left({\bf {n}}_{{\rm {td}}}^{{\rm {T}}}\left({-{\bf {I}}\!+\!{{\bf {n}}_{{\rm {t}}}}{\bf {n}}_{{\rm {t}}}^{{\rm {T}}}}\right){{\bf {n}}_{{\rm {t}}}} \!-\! \frac{1}{{\left\| {{{\bf{p}}_{\rm{r}}}} \right\|}}{\bf{n}}_{{\rm{td}}}^{\rm{T}}\left( { - {\bf{I}} \!+\! {{\bf{n}}_{\rm{t}}}{\bf{n}}_{\rm{t}}^{\rm{T}}} \right){{\bf{z}}_2}\right)  \\
		 &\quad - \frac{{{z_1}}}{{k_{\rm{b}}^2 \!-\! z_1^2}}{\left( {{{\bf{n}}_{{\rm{td}}}} \!\times\! {{\bf{n}}_{\rm{t}}}} \right)^{\rm{T}}}{}^{\rm{e}}\omega \\
		 & =\frac{{{z_1}}}{{k_{\rm{b}}^2 \!-\! z_1^2}} \left({\bf {n}}_{{\rm {td}}}^{{\rm {T}}}\left(-{{\bf {n}}_{{\rm {t}}}}+{{\bf {n}}_{{\rm {t}}}}{\bf {n}}_{{\rm {t}}}^{{\rm {T}}}{{\bf {n}}_{{\rm {t}}}} \right) \!-\! \frac{1}{{\left\| {{{\bf{p}}_{\rm{r}}}} \right\|}}{\bf{n}}_{{\rm{td}}}^{\rm{T}}\left( { - {\bf{I}} \!+\! {{\bf{n}}_{\rm{t}}}{\bf{n}}_{\rm{t}}^{\rm{T}}} \right){{\bf{z}}_2}\right)   \\
		 &\quad - \frac{{{z_1}}}{{k_{\rm{b}}^2 \!-\! z_1^2}}{\left( {{{\bf{n}}_{{\rm{td}}}} \!\times\! {{\bf{n}}_{\rm{t}}}} \right)^{\rm{T}}}{}^{\rm{e}}\omega \\
		 & = -\frac{{{z_1}}}{{k_{\rm{b}}^2 \!-\! z_1^2}} \frac{1}{{\left\| {{{\bf{p}}_{\rm{r}}}} \right\|}}{\bf{n}}_{{\rm{td}}}^{\rm{T}}\left( { - {\bf{I}} \!+\! {{\bf{n}}_{\rm{t}}}{\bf{n}}_{\rm{t}}^{\rm{T}}} \right){{\bf{z}}_2} \!-\! \frac{{{z_1}}}{{k_{\rm{b}}^2 \!-\! z_1^2}}{\left( {{{\bf{n}}_{{\rm{td}}}} \!\times\! {{\bf{n}}_{\rm{t}}}} \right)^{\rm{T}}}{}^{\rm{e}}\omega.
		\end{aligned}
	\end{small}
	\label{eq:L1_dot_appendix_rewritten}
\end{equation}}

{\subsubsection{Derivative of $L_3$}
Combining Eqs. \eqref{eq:L2_dot_rewritten}, \eqref{eq:L1_dot_appendix_rewritten}, the derivative of Eq. \eqref{eq:L3} is
\begin{equation}
\begin{aligned}
	{{\dot{L}}_{3}} &= -\frac{{{z_1}}}{{k_{\rm{b}}^2 - z_1^2}}{\left( {{{\bf{n}}_{{\rm{td}}}} \!\times\! {{\bf{n}}_{\rm{t}}}} \right)^{\rm{T}}}{}^{\rm{e}}\omega - k_1 {\bf {p}}_{\rm{r}}^{\rm{T}} {\bf {p}}_{\rm{r}} + {\bf {z}}_{2}^{\rm{T}} \left({\bf {a}}_{\rm{r}} + k_1 {\bf {v}}_{\rm{r}} \right) \\
	&= -\frac{{{z_1}}}{{k_{\rm{b}}^2 - z_1^2}}{\left( {{{\bf{n}}_{{\rm{td}}}} \!\times\! {{\bf{n}}_{\rm{t}}}} \right)^{\rm{T}}}{}^{\rm{e}}\omega - k_1 {\bf {p}}_{\rm{r}}^{\rm{T}} {\bf {p}}_{\rm{r}} - k_2 {\bf {z}}_{2}^{\rm{T}} {\bf {z}}_{2}. 
\end{aligned}
\label{eq:L3_dot_appendix}
\end{equation}}

{\subsubsection{Proof of Theorem 1}
A total Lyapunov candidate function is designed as:
\begin{equation}
	\begin{aligned}
		L &= L_3 + L_4 \\
		  &= 1 \!-\! {\bf{n}}_{{\rm{td}}}^{{\rm {T}}}{{\bf {n}}_{{\rm {t}}}}  +  \frac{1}{2} {\bf {p}}_{\rm{r}}^{\rm{T}} {\bf {p}}_{\rm{r}}  +  \frac{1}{2} {\bf {z}}_{2}^{\rm{T}} {\bf {z}}_{2}   {\mathop{{\rm tr}}\nolimits}\left({{\bf {I}}-{\bf {R}}_{{\rm {d}}}^{{\rm {T}}}{\bf {R}}_{{\rm {b}}}^{{\rm {e}}}}\right).
	\end{aligned}
\label{eq:L_appendix}
\end{equation}
Combining Eqs. \eqref{eq:L4_dot}, \eqref{eq:L3_dot_appendix}, the derivative of Eq. \eqref{eq:L_appendix} is
\begin{equation}
	\begin{small}
	\begin{aligned}
		{\dot{L}} 
		&= - k_1 {\bf {p}}_{\rm{r}}^{\rm{T}} {\bf {p}}_{\rm{r}} - k_2 {\bf {z}}_{2}^{\rm{T}} {\bf {z}}_{2} \\ &\quad -\frac{{{z_1}}}{{k_{\rm{b}}^2 - z_1^2}}{\left( {{{\bf{n}}_{{\rm{td}}}} \!\times\! {{\bf{n}}_{\rm{t}}}} \right)^{\rm{T}}}{}^{\rm{e}}\omega +  {\mathop{{\rm vex}}\nolimits}\left({{\bf {R}}_{{\rm {d}}}^{{\rm {T}}}{\bf {R}}_{{\rm {b}}}^{{\rm {e}}}-{\bf {R}}{{_{{\rm {b}}}^{{\rm {e}}}}^{{\rm {T}}}}{{\bf {R}}_{{\rm {d}}}}}\right)^{\rm{T}} {}^{{\rm {b}}}{\bf {\omega}} \\
		&= - k_1 {\bf {p}}_{\rm{r}}^{\rm{T}} {\bf {p}}_{\rm{r}} - k_2 {\bf {z}}_{2}^{\rm{T}} {\bf {z}}_{2} \\ &\quad - \left( \frac{{{z_1}}}{{k_{\rm{b}}^2 \!-\! z_1^2}} \left({{\bf{n}}_{{\rm{td}}} \!\times\! {{\bf {n}}_{{\rm {t}}}}}\right)^{\rm{T}} {\bf {R}}_{{\rm {b}}}^{{\rm {e}}} \!-\!  {\mathop{{\rm vex}}\nolimits}\left({{\bf {R}}_{{\rm {d}}}^{{\rm {T}}}{\bf {R}}_{{\rm {b}}}^{{\rm {e}}}\!-\! {\bf {R}}{{_{{\rm {b}}}^{{\rm {e}}}}^{{\rm {T}}}}{{\bf {R}}_{{\rm {d}}}}}\right)^{\rm{T}} \right) {}^{{\rm {b}}}{\bf {\omega}} \\
		&= - k_1 {\bf {p}}_{\rm{r}}^{\rm{T}} {\bf {p}}_{\rm{r}} - k_2 {\bf {z}}_{2}^{\rm{T}} {\bf {z}}_{2}  -   {\bf {w}}^{\rm{T}} {\rm{sat}}\left({\bf {w}},\omega_{{\rm {m}}}\right)  \\
		&= - k_1 {\bf {p}}_{\rm{r}}^{\rm{T}} {\bf {p}}_{\rm{r}} - k_2 {\bf {z}}_{2}^{\rm{T}} {\bf {z}}_{2}  -  {\kappa_{\omega_{{\rm {m}}}}}\left({\bf {{\bf {w}}}}\right) {\bf {w}}^{\rm{T}} {\bf {w}}  \\
		&\le 0
	\end{aligned}
	\end{small}
	\label{eq:L_dot_appendix}
\end{equation}
where ${\bf {w}} \!=\! \frac{{{z_1}}}{{k_{\rm{b}}^2 \!-\! z_1^2}} {{\bf {R}}_{{\rm {b}}}^{{\rm {e}}}}^{\rm{T}} \left({{\bf{n}}_{{\rm{td}}} \!\times\! {{\bf {n}}_{{\rm {t}}}}}\right) \!-\!  {\mathop{{\rm vex}}\nolimits}\left({{\bf {R}}_{{\rm {d}}}^{{\rm {T}}}{\bf {R}}_{{\rm {b}}}^{{\rm {e}}}\!-\! {\bf {R}}{{_{{\rm {b}}}^{{\rm {e}}}}^{{\rm {T}}}}{{\bf {R}}_{{\rm {d}}}}}\right) $, and ${}^{{\rm {b}}}{\bf {\omega}} = {\rm{sat}}\left({\bf {w}},\omega_{{\rm {m}}}\right)$.}

{According to \textit{Lemma 1} in \cite{TEE2009918}, we have 
\[
{\bf{p}}_{\rm{r}} \to {\bf{0}}, \quad {\bf {n}}_{{\rm {t}}} \in \mathcal{S} = \left\{ {\bf {x}} \mid \left\lvert {{\bf{n}}_{{\rm{td}}}^{{\rm {T}}} {\bf {x}}} \right\rvert > 1 - k_{\rm{b}} \right\},
\]
provided that the initial state satisfies ${\bf {n}}_{{\rm {t}}}\left(0\right)  \in \mathcal{S}$.
This ensures that the interceptor's relative position to the target converges to zero. Additionally, throughout the interception process, the target vector remains close to the designed LOS vector, ensuring the target remains locked within the interceptor's FOV.}

\subsection{DKF Propagation with IMU Measurements\label{subsec:appendix-DKF-Propagation}}

\subsubsection{Quaternion Status Transfer}

The discrete-time recurrence relation of state ${{\bf {q}}}$ is
\begin{equation}
{\bf {q}}_{k}={\bf {q}}_{k-1}\otimes\Delta{\bf {q}}_{k-1}={\bf {M}}\left({\Delta{\bf {q}}_{k-1}}\right){\bf {q}}_{k-1}
\end{equation}
where ${\bf {M}}\left({\Delta{\bf {q}}_{k-1}}\right)$ is given in
\textit{Appendix \ref{subsec:appendix-Matrices}}, ${^{{\rm {b}}}{\bf {\omega}}}$
is the estimated value at time $t_{k}$, and the relationship with
the measured value of the gyroscope is Eq. \eqref{eq:gyr-measurement}.

So, ${\bf {F}}_{{\bf {q}}_{k-1}}^{{\bf {q}}_{k}}={\bf {M}}\left({\Delta{\bf {q}}_{k-1}}\right)$.
Because the estimated value and the bias of gyroscope have
the same order and coefficient, their partial derivatives have the
same form:
\begin{equation}
\begin{aligned}{\bf {F}}_{{{\bf {b}}_{{\rm {gyr}},k-1}}}^{{\bf {q}}_{k}} & =\left[{\begin{array}{ccc}
{\frac{{q_{1}}}{2}} & {\frac{{q_{2}}}{2}} & {\frac{{q_{3}}}{2}}\\
{-\frac{{q_{0}}}{2}} & {\frac{{q_{3}}}{2}} & {-\frac{{q_{2}}}{2}}\\
{-\frac{{q_{3}}}{2}} & {-\frac{{q_{0}}}{2}} & {\frac{{q_{1}}}{2}}\\
{\frac{{q_{2}}}{2}} & {-\frac{{q_{1}}}{2}} & {-\frac{{q_{0}}}{2}}
\end{array}}\right]\Delta t,\\
{\bf {G}}_{^{{\rm {b}}}{\bf {\omega}}_{k-1}}^{{\bf {q}}_{k}} & ={\bf {F}}_{{{\bf {b}}_{{\rm {gyr}},k-1}}}^{{\bf {q}}_{k}}.
\end{aligned}
\end{equation}

\subsubsection{Position and Velocity Status Transfer}

For low-cost integrated navigation systems embedded in small multicopters, relative positions are generally used for navigation, while the influence of coning and sculling errors is ignored \cite{roscoe2001equivalency}.
The position transfer equation is written as
\begin{equation}
	{{\bf {p}}}_{{\rm{r}},{k}}={{\bf {p}}}_{{\rm{r}},{k-1}}+\frac{1}{2}\left({{{\bf {v}}}_{{\rm{r}},{k-1}}+{{\bf {v}}}_{{\rm{r}},{k}}}\right)\Delta t.
\end{equation}

The corresponding state transition matrix is
\begin{equation}
{\bf {F}}_{{{\bf {v}}}_{{\rm{r}},{k-1}}}^{{{\bf {p}}}_{{\rm{r}},{k}}}={\bf {I}}\,\Delta t.
\end{equation}
Ignoring the angular velocity of the earth\textquoteright s rotation,
the velocity Eq. \eqref{eq:ve-and-higher-order-state} is discretized
into a transfer equation, and we can get
\begin{equation}
	{{\bf {v}}}_{{\rm{r}},{k}}={{\bf {v}}}_{{\rm{r}},{k-1}}+\left({\bf {R}}_{{\rm {b}}}^{{\rm {e}}}\cdot{}^{{\rm {b}}}{\bf {a}}_{k-1}+g{{\bf {e}}_{3}}\right) \Delta t
\end{equation}
where $\Delta t$ is the period of the discrete system. Because there
is a coordinate system transformation between the acceleration measurement
in BCS and the velocity state in EFCS, there is a transfer relationship between the acceleration
and the orientation quaternion, that is
\begin{equation}
{\bf {F}}_{{\bf {q}}_{k-1}}^{{{\bf {v}}}_{{\rm{r}},{k}}}=2{\left[{\begin{array}{ccc}
{{{\bf {M}}_{1}}{}^{{\rm {b}}}{\bf {a}}} & {{{\bf {M}}_{2}}{}^{{\rm {b}}}{\bf {a}}} & {{{\bf {M}}_{3}}{}^{{\rm {b}}}{\bf {a}}}\end{array}}\right]^{{\rm {T}}}}\Delta t
\end{equation}
where $^{{\rm {b}}}{\bf {a}}=\left[{\begin{array}{ccc}
{a_{{x_{{\rm {b}}}}}} & {a_{{y_{{\rm {b}}}}}} & {a_{{z_{{\rm {b}}}}}}\end{array}}\right]^{{\rm {T}}}={{\bf {a}}_{{\rm {acc}}}}-{{\bf {b}}_{{\rm {acc}}}}$, ${{{\bf {M}}_{1}}}$, ${{{\bf {M}}_{2}}}$,
${{{\bf {M}}_{3}}}$ is given in \textit{Appendix \ref{subsec:appendix-Matrices}}.

Due to the estimated value and the bias of accelerometer have
the same order and coefficient, their partial derivatives have the
same form:
\begin{equation}
\begin{aligned}{\bf {F}}_{{{\bf {b}}_{{\rm {acc,}}k-1}}}^{{{\bf {v}}}_{{\rm{r}},{k}}} & =-{\bf {R}}_{{\rm {b}}}^{{\rm {e}}}\Delta t,\\
{\bf {G}}_{^{{\rm {e}}}{\bf {a}}_{k-1}}^{{{\bf {v}}}_{{\rm{r}},{k}}} & ={\bf {F}}_{{{\bf {b}}_{{\rm {acc,}}k-1}}}^{{{\bf {v}}}_{{\rm{r}},{k}}}.
\end{aligned}
\end{equation}

\subsubsection{Image Measurement State Transfer}

The IBVS Eq. \eqref{eq:IBVS} describes the relationship between the change of image feature points and the twist in CCS. 
Converting the twist in the equation from CCS to BCS and discretizing it, we get the transfer equation of image feature as
\begin{equation}
\setlength{\arraycolsep}{3pt}\begin{aligned}^{{\rm {i}}}{\bf {\bar{p}}}_{k}= & ^{{\rm {i}}}{\bf {\bar{p}}}_{k-1}+\left[{\begin{array}{ccc}
{-\frac{1}{{p_{{z_{{\rm {c}}}}}}}} & 0 & {\frac{{{\bar{p}}_{{x_{{\rm {i}}}}}}}{{p_{{z_{{\rm {c}}}}}}}}\\
0 & {-\frac{1}{{p_{{z_{{\rm {c}}}}}}}} & {\frac{{{\bar{p}}_{{y_{{\rm {i}}}}}}}{{p_{{z_{{\rm {c}}}}}}}}
\end{array}}\right]{\bf {R}}{_{{\rm {c}}}^{{\rm {b}}}}^{{\rm {T}}}{\bf {R}}{_{{\rm {b}}}^{{\rm {e}}}}^{{\rm {T}}}{{\bf {v}}}_{{\rm{r}},{k-1}}\Delta t\\
 & +\left[{\begin{array}{ccc}
{{{\bar{p}}_{{x_{{\rm {i}}}}}}{{\bar{p}}_{{y_{{\rm {i}}}}}}} & {-\left({1+\bar{p}_{{x_{{\rm {i}}}}}^{2}}\right)} & {{\bar{p}}_{{y_{{\rm {i}}}}}}\\
{1+\bar{p}_{{y_{{\rm {i}}}}}^{2}} & {-{{\bar{p}}_{{x_{{\rm {i}}}}}}{{\bar{p}}_{{y_{{\rm {i}}}}}}} & {-{{\bar{p}}_{{x_{{\rm {i}}}}}}}
\end{array}}\right]{\bf {R}}{_{{\rm {c}}}^{{\rm {b}}}}^{{\rm {T}}}{}^{{\rm {b}}}{\bf {\omega}}_{k-1}\Delta t.
\end{aligned}
\end{equation}

The state transition matrix of the image feature $^{{\rm {i}}}{\bf {\bar{p}}}_{k}$
to the orientation quaternion ${\bf {q}}_{k-1}$ is
\begin{equation}
{\bf {F}}_{{\bf {q}}_{k-1}}^{^{{\rm {i}}}{\bf {\bar{p}}}_{k}}=\frac{1}{{p_{{z_{{\rm {c}}}}}}}{\left[{\begin{array}{cc}
{{{\bf {M}}_{4}}{{\bf {v}}}_{{\rm{r}}}} & {{{\bf {M}}_{5}}{{\bf {v}}}_{{\rm{r}}}}\end{array}}\right]^{{\rm {T}}}}\Delta t\label{eq:F-p-q}
\end{equation}
where ${{{\bf {M}}_{4}}}$, ${{{\bf {M}}_{5}}}$
is given in \textit{Appendix \ref{subsec:appendix-Matrices}}.

The state transition matrix of the image feature $^{{\rm {i}}}{\bf {\bar{p}}}_{k}$
to the velocity ${{\bf {v}}}_{{\rm{r}},{k-1}}$ is
\begin{equation}
{\bf {F}}_{{{\bf {v}}}_{{\rm{r}},{k-1}}}^{^{{\rm {i}}}{\bf {\bar{p}}}_{k}}=\left[{\begin{array}{ccc}
{-\frac{1}{{p_{{z_{{\rm {c}}}}}}}} & 0 & {\frac{{{\bar{p}}_{{x_{{\rm {i}}}}}}}{{p_{{z_{{\rm {c}}}}}}}}\\
0 & {-\frac{1}{{p_{{z_{{\rm {c}}}}}}}} & {\frac{{{\bar{p}}_{{y_{{\rm {i}}}}}}}{{p_{{z_{{\rm {c}}}}}}}}
\end{array}}\right]{\bf {R}}{_{{\rm {c}}}^{{\rm {b}}}}^{{\rm {T}}}{\bf {R}}{_{{\rm {b}}}^{{\rm {e}}}}^{{\rm {T}}}\Delta t.\label{eq:F-p-v}
\end{equation}

The state transition matrix of the image feature $^{{\rm {i}}}{\bf {\bar{p}}}_{k}$
of the $k$-th frame to the image feature ${^{{\rm {i}}}{\bf {\bar{p}}}_{k-1}}$
of the $\left(k-1\right)$-th frame is
\begin{equation}\label{eq:F-p-p}
\begin{small}
\setlength{\arraycolsep}{0pt}{\bf {F}}_{^{{\rm {i}}}{\bf {\bar{p}}}_{k-1}}^{^{{\rm {i}}}{\bf {\bar{p}}}_{k}}\!={\bf {I}}+\left[\!{\begin{array}{cc}
{\frac{{v_{{z_{{\rm {c}}}}}}}{{p_{{z_{{\rm {c}}}}}}}\!+\!{{\bar{p}}_{{y_{{\rm {i}}}}}}{\omega_{{x_{{\rm {c}}}}}}\!-\!2{{\bar{p}}_{{x_{{\rm {i}}}}}}{\omega_{{y_{{\rm {c}}}}}}} & {{{\bar{p}}_{{x_{{\rm {i}}}}}}{\omega_{{x_{{\rm {c}}}}}}\!+\!{\omega_{{z_{{\rm {c}}}}}}}\\
{-\!{{\bar{p}}_{{y_{{\rm {i}}}}}}{\omega_{{y_{{\rm {c}}}}}}\!-\!{\omega_{{z_{{\rm {c}}}}}}} & {\frac{{v_{{z_{{\rm {c}}}}}}}{{p_{{z_{{\rm {c}}}}}}}\!+\!2{{\bar{p}}_{{y_{{\rm {i}}}}}}{\omega_{{x_{{\rm {c}}}}}}\!-\!{{\bar{p}}_{{x_{{\rm {i}}}}}}{\omega_{{y_{{\rm {c}}}}}}}
\end{array}}\!\right]\!\Delta t.
\end{small}
\end{equation}

The state transition matrix and the system noise driving matrix of
the image feature $^{{\rm {i}}}{\bf {\bar{p}}}_{k}$ to the gyroscope
bias ${{{\bf {b}}_{{\rm {gyr}},k-1}}}$ are
\begin{equation}
\begin{aligned}{\bf {F}}_{{{\bf {b}}_{{\rm {gyr}},k-1}}}^{^{{\rm {i}}}{\bf {\bar{p}}}_{k}} & =-\left[{\begin{array}{ccc}
{{{\bar{p}}_{{x_{{\rm {i}}}}}}{{\bar{p}}_{{y_{{\rm {i}}}}}}} & {-\left({1+\bar{p}_{{x_{{\rm {i}}}}}^{2}}\right)} & {{\bar{p}}_{{y_{{\rm {i}}}}}}\\
{1+\bar{p}_{{y_{{\rm {i}}}}}^{2}} & {-{{\bar{p}}_{{x_{{\rm {i}}}}}}{{\bar{p}}_{{y_{{\rm {i}}}}}}} & {-{{\bar{p}}_{{x_{{\rm {i}}}}}}}
\end{array}}\right]{\bf {R}}{_{{\rm {c}}}^{{\rm {b}}}}^{{\rm {T}}},\\
{\bf {G}}_{^{{\rm {b}}}{\bf {\omega}}_{k-1}}^{^{{\rm {i}}}{\bf {\bar{p}}}_{k}} & ={\bf {F}}_{{{\bf {b}}_{{\rm {gyr}},k-1}}}^{^{{\rm {i}}}{\bf {\bar{p}}}_{k}}.
\end{aligned}
\label{eq:F-p-b}
\end{equation}

The Eqs. \eqref{eq:F-p-q}--\eqref{eq:F-p-b} complete the state transition
matrix ${{\bf {F}}_{k}}$ and the system noise driving
matrix ${{\bf {G}}_{k}}$, and the state propagation can
be realized according to the Eqs. \eqref{eq:x_hat}--\eqref{eq:P_k_k-1}.

\subsection{State Recurrence Equations\label{subsec:appendix-Matrices}}
\begin{small}
\[
\setlength{\arraycolsep}{3.5pt}
\protect\begin{aligned}{\bf {M}}\left({\Delta{\bf {q}}\left({k-1}\right)}\right) = \left[{\protect\begin{array}{cccc}
1 & {-\frac{{{\omega_{{x_{{\rm {b}}}}}}\Delta t}}{2}} & {-\frac{{{\omega_{{y_{{\rm {b}}}}}}\Delta t}}{2}} & {-\frac{{{\omega_{{z_{{\rm {b}}}}}}\Delta t}}{2}}\protect\\
{\frac{{{\omega_{{x_{{\rm {b}}}}}}\Delta t}}{2}} & 1 & {\frac{{{\omega_{{z_{{\rm {b}}}}}}\Delta t}}{2}} & {-\frac{{{\omega_{{y_{{\rm {b}}}}}}\Delta t}}{2}}\protect\\
{\frac{{{\omega_{{y_{{\rm {b}}}}}}\Delta t}}{2}} & {-\frac{{{\omega_{{z_{{\rm {b}}}}}}\Delta t}}{2}} & 1 & {\frac{{{\omega_{{x_{{\rm {b}}}}}}\Delta t}}{2}}\protect\\
{\frac{{{\omega_{{z_{{\rm {b}}}}}}\Delta t}}{2}} & {\frac{{{\omega_{{y_{{\rm {b}}}}}}\Delta t}}{2}} & {-\frac{{{\omega_{{x_{{\rm {b}}}}}}\Delta t}}{2}} & 1
\protect\end{array}}\right],
\protect\end{aligned}
\]
\end{small}
\begin{small}
\[
\setlength{\arraycolsep}{1.0pt}
{{\bf {M}}_{1}} \!=\!\! \left[{\begin{array}{ccc}
{q_{0}} & {-{q_{3}}} & {q_{2}}\\
{q_{1}} & {q_{2}} & {q_{3}}\\
{-{q_{2}}} & {q_{1}} & {q_{0}}\\
{-{q_{3}}} & {-{q_{0}}} & {q_{1}}
\end{array}}\right]\!,
{{\bf {M}}_{2}} \!=\!\! \left[{\begin{array}{ccc}
	{q_{3}} & {q_{0}} & {-{q_{1}}}\\
	{q_{2}} & {-{q_{1}}} & {-{q_{0}}}\\
	{q_{1}} & {q_{2}} & {q_{3}}\\
	{q_{0}} & {-{q_{3}}} & {q_{2}}
	\end{array}}\right]\!,
{{\bf {M}}_{3}} \!=\!\! \left[{\begin{array}{ccc}
	{-{q_{2}}} & {q_{1}} & {q_{0}}\\
	{q_{3}} & {q_{0}} & {-{q_{1}}}\\
	{-{q_{0}}} & {q_{3}} & {-{q_{2}}}\\
	{q_{1}} & {q_{2}} & {q_{3}}
	\end{array}}\right]\!,
\]
\end{small}
\begin{small}
\[
{{\bf {M}}_{4}}=\left[{\begin{array}{ccc}
{2{{\bar{p}}_{{x_{{\rm {i}}}}}}{q_{0}}+2{q_{3}}} & {2{{\bar{p}}_{{x_{{\rm {i}}}}}}{q_{3}}-2{q_{0}}} & {-2{{\bar{p}}_{{x_{{\rm {i}}}}}}{q_{2}}-2{q_{1}}}\\
{2{{\bar{p}}_{{x_{{\rm {i}}}}}}{q_{1}}-2{q_{2}}} & {2{{\bar{p}}_{{x_{{\rm {i}}}}}}{q_{2}}+2{q_{1}}} & {2{{\bar{p}}_{{x_{{\rm {i}}}}}}{q_{3}}-2{q_{0}}}\\
{2{{\bar{p}}_{{x_{{\rm {i}}}}}}{q_{2}}-2{q_{1}}} & {2{{\bar{p}}_{{x_{{\rm {i}}}}}}{q_{1}}-2{q_{2}}} & {-2{{\bar{p}}_{{x_{{\rm {i}}}}}}{q_{0}}-2{q_{3}}}\\
{2{{\bar{p}}_{{x_{{\rm {i}}}}}}{q_{3}}+2{q_{0}}} & {2{{\bar{p}}_{{x_{{\rm {i}}}}}}{q_{0}}+2{q_{3}}} & {2{{\bar{p}}_{{x_{{\rm {i}}}}}}{q_{1}}-2{q_{2}}}
\end{array}}\right],
\]
\end{small}
\begin{small}
\[
{{\bf {M}}_{5}}=\left[{\begin{array}{ccc}
{2{{\bar{p}}_{{y_{{\rm {i}}}}}}{q_{0}}-2{q_{2}}} & {2{{\bar{p}}_{{y_{{\rm {i}}}}}}{q_{3}}+2{q_{1}}} & {-2{{\bar{p}}_{{y_{{\rm {i}}}}}}{q_{2}}-2{q_{0}}}\\
{2{{\bar{p}}_{{y_{{\rm {i}}}}}}{q_{1}}-2{q_{3}}} & {2{{\bar{p}}_{{y_{{\rm {i}}}}}}{q_{2}}+2{q_{0}}} & {2{{\bar{p}}_{{y_{{\rm {i}}}}}}{q_{3}}+2{q_{1}}}\\
{2{{\bar{p}}_{{y_{{\rm {i}}}}}}{q_{2}}-2{q_{0}}} & {2{{\bar{p}}_{{y_{{\rm {i}}}}}}{q_{1}}-2{q_{3}}} & {-2{{\bar{p}}_{{y_{{\rm {i}}}}}}{q_{0}}+2{q_{2}}}\\
{2{{\bar{p}}_{{y_{{\rm {i}}}}}}{q_{3}}-2{q_{1}}} & {2{{\bar{p}}_{{y_{{\rm {i}}}}}}{q_{0}}-2{q_{2}}} & {2{{\bar{p}}_{{y_{{\rm {i}}}}}}{q_{1}}-2{q_{3}}}
\end{array}}\right].
\]
\end{small}

\bibliography{myreferences}

\begin{thebibliography}{10}
\providecommand{\url}[1]{#1}
\csname url@samestyle\endcsname
\providecommand{\newblock}{\relax}
\providecommand{\bibinfo}[2]{#2}
\providecommand{\BIBentrySTDinterwordspacing}{\spaceskip=0pt\relax}
\providecommand{\BIBentryALTinterwordstretchfactor}{4}
\providecommand{\BIBentryALTinterwordspacing}{\spaceskip=\fontdimen2\font plus
\BIBentryALTinterwordstretchfactor\fontdimen3\font minus
  \fontdimen4\font\relax}
\providecommand{\BIBforeignlanguage}[2]{{%
\expandafter\ifx\csname l@#1\endcsname\relax
\typeout{** WARNING: IEEEtran.bst: No hyphenation pattern has been}%
\typeout{** loaded for the language `#1'. Using the pattern for}%
\typeout{** the default language instead.}%
\else
\language=\csname l@#1\endcsname
\fi
#2}}
\providecommand{\BIBdecl}{\relax}
\BIBdecl

\bibitem{Lee2021Antisway}
S.~Lee and H.~Son, ``Antisway control of a multirotor with cable-suspended
  payload,'' \emph{IEEE Transactions on Control Systems Technology}, vol.~29,
  no.~6, pp. 2630--2638, 2021.

\bibitem{floreano2015science}
D.~Floreano and R.~J. Wood, ``Science, technology and the future of small
  autonomous drones,'' \emph{Nature}, vol. 521, no. 7553, pp. 460--466, 2015.

\bibitem{ritchie2017micro}
M.~Ritchie, F.~Fioranelli, and H.~Borrion, ``Micro {UAV} crime prevention: Can
  we help {Princess Leia}?'' in \emph{Crime Prevention in the 21st
  Century}.\hskip 1em plus 0.5em minus 0.4em\relax Springer, 2017, pp.
  359--376.

\bibitem{shi2018anti}
X.~Shi, C.~Yang, W.~Xie, C.~Liang, Z.~Shi, and J.~Chen, ``Anti-drone system
  with multiple surveillance technologies: Architecture, implementation, and
  challenges,'' \emph{IEEE Communications Magazine}, vol.~56, no.~4, pp.
  68--74, 2018.

\bibitem{10.1145/3309735}
J.~Noh, Y.~Kwon, Y.~Son, H.~Shin, D.~Kim, J.~Choi, and Y.~Kim, ``Tractor beam:
  Safe-hijacking of consumer drones with adaptive {GPS} spoofing,'' \emph{ACM
  Transactions on Privacy and Security}, vol.~22, no.~2, pp. 1--26, 2019.

\bibitem{parlin2018jamming}
K.~P{\"a}rlin, M.~M. Alam, and Y.~Le~Moullec, ``Jamming of {UAV} remote control
  systems using software defined radio,'' in \emph{2018 International
  Conference on Military Communications and Information Systems
  (ICMCIS)}.\hskip 1em plus 0.5em minus 0.4em\relax IEEE, 2018, pp. 1--6.

\bibitem{Dronecatcher2021}
D.~Dynamics, ``Dronecatcher,'' www.dronecatcher.nl, 2021.

\bibitem{chaari2020testing}
M.~Z. Chaari and S.~Al-Maadeed, ``Testing the efficiency of laser technology to
  destroy the rogue drones,'' \emph{Security and Defence Quarterly}, vol.~32,
  no.~5, pp. 31--38, 2020.

\bibitem{chen2020aerobatic}
Y.~Chen and N.~O. Pérez-Arancibia, ``Controller synthesis and performance
  optimization for aerobatic quadrotor flight,'' \emph{IEEE Transactions on
  Control Systems Technology}, vol.~28, no.~6, pp. 2204--2219, 2020.

\bibitem{foehn2022alphapilot}
P.~Foehn, D.~Brescianini, E.~Kaufmann, T.~Cieslewski, M.~Gehrig, M.~Muglikar,
  and D.~Scaramuzza, ``Alphapilot: Autonomous drone racing,'' \emph{Autonomous
  Robots}, vol.~46, no.~1, pp. 307--320, 2022.

\bibitem{yang2020autonomous}
K.~Yang and Q.~Quan, ``An autonomous intercept drone with image-based visual
  servo,'' in \emph{2020 IEEE International Conference on Robotics and
  Automation (ICRA)}, 2020, pp. 2230--2236.

\bibitem{stasinchuk2021multi}
Y.~Stasinchuk, M.~Vrba, M.~Petrl{\'\i}k, T.~B{\'a}{\v{c}}a, V.~Spurn{\`y},
  D.~Hert, D.~{\v{Z}}aitl{\'\i}k, T.~Nascimento, and M.~Saska, ``A multi-{UAV}
  system for detection and elimination of multiple targets,'' in \emph{2021
  IEEE International Conference on Robotics and Automation (ICRA)}.\hskip 1em
  plus 0.5em minus 0.4em\relax IEEE, 2021, pp. 555--561.

\bibitem{vrba2022autonomous}
M.~Vrba, Y.~Stasinchuk, T.~B{\'a}{\v{c}}a, V.~Spurn{\`y}, M.~Petrl{\'\i}k,
  D.~He{\v{r}}t, D.~{\v{Z}}aitl{\'\i}k, and M.~Saska, ``Autonomous capture of
  agile flying objects using {UAVs}: The {MBZIRC} 2020 challenge,''
  \emph{Robotics and Autonomous Systems}, vol. 149, p. 103970, 2022.

\bibitem{Lin2022}
H.~Tao, D.~Lin, S.~He, T.~Song, and R.~Jin, ``Optimal terminal-velocity-control
  guidance for intercepting non-cooperative maneuvering quadcopter,''
  \emph{Journal of Field Robotics}, vol.~39, no.~4, pp. 457--472, 2022.

\bibitem{Deep2020}
A.~Loquercio, E.~Kaufmann, R.~Ranftl, A.~Dosovitskiy, V.~Koltun, and
  D.~Scaramuzza, ``Deep drone racing: From simulation to reality with domain
  randomization,'' \emph{IEEE Transactions on Robotics}, vol.~36, no.~1, pp.
  1--14, 2020.

\bibitem{fu2022learning}
J.~Fu, Y.~Song, Y.~Wu, F.~Yu, and D.~Scaramuzza, ``Learning deep sensorimotor
  policies for vision-based autonomous drone racing,'' \emph{arXiv preprint
  arXiv:2210.14985}, 2022.

\bibitem{9946840}
J.~Li, Z.~Ning, S.~He, C.-H. Lee, and S.~Zhao, ``Three-dimensional bearing-only
  target following via observability-enhanced helical guidance,'' \emph{IEEE
  Transactions on Robotics}, vol.~39, no.~2, pp. 1509--1526, 2023.

\bibitem{hanover2023autonomous}
D.~Hanover, A.~Loquercio, L.~Bauersfeld, A.~Romero, R.~Penicka, Y.~Song,
  G.~Cioffi, E.~Kaufmann, and D.~Scaramuzza, ``Autonomous drone racing: A
  survey,'' \emph{arXiv e-prints}, pp. arXiv--2301, 2023.

\bibitem{8746787}
W.~Zhao, H.~Liu, F.~L. Lewis, K.~P. Valavanis, and X.~Wang, ``Robust visual
  servoing control for ground target tracking of quadrotors,'' \emph{IEEE
  Transactions on Control Systems Technology}, vol.~28, no.~5, pp. 1980--1987,
  2020.

\bibitem{8628313}
H.~Xie, A.~F. Lynch, K.~H. Low, and S.~Mao, ``Adaptive output-feedback
  image-based visual servoing for quadrotor unmanned aerial vehicles,''
  \emph{IEEE Transactions on Control Systems Technology}, vol.~28, no.~3, pp.
  1034--1041, 2020.

\bibitem{10044962}
G.~Wang, J.~Qin, Q.~Liu, Q.~Ma, and C.~Zhang, ``Image-based visual servoing of
  quadrotors to arbitrary flight targets,'' \emph{IEEE Robotics and Automation
  Letters}, vol.~8, no.~4, pp. 2022--2029, 2023.

\bibitem{7565735}
J.~M. Pak, C.~K. Ahn, P.~Shi, Y.~S. Shmaliy, and M.~T. Lim, ``Distributed
  hybrid particle/{FIR} filtering for mitigating {NLOS} effects in {TOA}-based
  localization using wireless sensor networks,'' \emph{IEEE Transactions on
  Industrial Electronics}, vol.~64, no.~6, pp. 5182--5191, 2017.

\bibitem{GUAN2018161}
R.~P. Guan, B.~Ristic, L.~Wang, and R.~Evans, ``{Monte Carlo} localisation of a
  mobile robot using a {Doppler-Azimuth} radar,'' \emph{Automatica}, vol.~97,
  pp. 161--166, 2018.

\bibitem{9477131}
H.~Zhu, J.~Mi, Y.~Li, K.-V. Yuen, and H.~Leung, ``{VB}-{Kalman} based
  localization for connected vehicles with delayed and lost measurements:
  Theory and experiments,'' \emph{IEEE/ASME Transactions on Mechatronics},
  vol.~27, no.~3, pp. 1370--1378, 2022.

\bibitem{9362168}
K.~Gamagedara, T.~Lee, and M.~Snyder, ``Quadrotor state estimation with {IMU}
  and delayed real-time kinematic {GPS},'' \emph{IEEE Transactions on Aerospace
  and Electronic Systems}, vol.~57, no.~5, pp. 2661--2673, 2021.

\bibitem{KAI20178189}
J.-M. Kai, G.~Allibert, M.-D. Hua, and T.~Hamel, ``Nonlinear feedback control
  of quadrotors exploiting first-order drag effects,''
  \emph{IFAC-PapersOnLine}, vol.~50, no.~1, pp. 8189--8195, 2017, 20th IFAC
  World Congress.

\bibitem{scirobotics.adg1462}
Y.~Song, A.~Romero, M.~M{\"u}ller, V.~Koltun, and D.~Scaramuzza, ``Reaching the
  limit in autonomous racing: Optimal control versus reinforcement learning,''
  \emph{Science Robotics}, vol.~8, no.~82, p. eadg1462, 2023.

\bibitem{corke2011robotics}
P.~I. Corke and O.~Khatib, \emph{Robotics, Vision and Control Fundamental
  Algorithms in {MATLAB}}.\hskip 1em plus 0.5em minus 0.4em\relax Berlin,
  Heidelberg: Springer, 2011, vol.~73.

\bibitem{chaumette2006visual}
F.~Chaumette and S.~Hutchinson, ``Visual servo control. {I}. {Basic}
  approaches,'' \emph{IEEE Robotics \& Automation Magazine}, vol.~13, no.~4,
  pp. 82--90, 2006.

\bibitem{yu2015high}
Y.~Yu, S.~Yang, M.~Wang, C.~Li, and Z.~Li, ``High performance full attitude
  control of a quadrotor on {SO(3)},'' in \emph{2015 IEEE International
  Conference on Robotics and Automation (ICRA)}, 2015, pp. 1698--1703.

\bibitem{murray2017mathematical}
R.~M. Murray, Z.~Li, and S.~S. Sastry, \emph{A Mathematical Introduction to
  Robotic Manipulation}.\hskip 1em plus 0.5em minus 0.4em\relax CRC Press,
  2017.

\bibitem{lee2010geometric}
T.~Lee, M.~Leok, and N.~H. McClamroch, ``Geometric tracking control of a
  quadrotor {UAV} on {SE(3)},'' in \emph{49th IEEE Conference on Decision and
  Control (CDC)}.\hskip 1em plus 0.5em minus 0.4em\relax IEEE, 2010, pp.
  5420--5425.

\bibitem{Fu2023}
R.~Fu, Q.~Quan, M.~Li, and K.-Y. Cai, ``Practical distributed control for
  cooperative multicopters in structured free flight concepts,'' \emph{IEEE
  Transactions on Intelligent Transportation Systems}, vol.~24, no.~4, pp.
  4203--4216, 2023.

\bibitem{dai2021rflysim}
X.~Dai, C.~Ke, Q.~Quan, and K.-Y. Cai, ``{RFlySim}: Automatic test platform for
  {UAV} autopilot systems with {FPGA}-based hardware-in-the-loop simulations,''
  \emph{Aerospace Science and Technology}, vol. 114, p. 106727, 2021.

\bibitem{Shnidman1995CEP}
D.~Shnidman, ``Efficient computation of the circular error probability ({CEP})
  integral,'' \emph{IEEE Transactions on Automatic Control}, vol.~40, no.~8,
  pp. 1472--1474, 1995.

\bibitem{TEE2009918}
K.~P. Tee, S.~S. Ge, and E.~H. Tay, ``Barrier lyapunov functions for the
  control of output-constrained nonlinear systems,'' \emph{Automatica},
  vol.~45, no.~4, pp. 918--927, 2009.

\bibitem{roscoe2001equivalency}
K.~M. Roscoe, ``Equivalency between strapdown inertial navigation coning and
  sculling integrals/algorithms,'' \emph{Journal of Guidance, Control, and
  Dynamics}, vol.~24, no.~2, pp. 201--205, 2001.

\end{thebibliography}
\bibliographystyle{IEEEtran}

\end{document}